\documentclass[letterpaper,11pt]{article}
\usepackage{etex}
\usepackage[margin=1in]{geometry}

\usepackage[dvipsnames,table,xcdraw]{xcolor}
\usepackage{amssymb,amsfonts,amsmath,amstext,amsthm,mathrsfs}
\usepackage{graphics,latexsym,epsfig,psfrag,wrapfig,comment,paralist}
\usepackage{xspace}
\usepackage{subcaption,bm,multirow}
\usepackage{booktabs}
\usepackage{mathdots}
\usepackage{bbold}
\usepackage{centernot}
\usepackage{algorithm}
\usepackage{algorithmic}
\usepackage{prettyref}
\usepackage{listings}
\usepackage{tikz}
\usetikzlibrary{decorations,calligraphy}
\usepackage{pgfplots}
\usetikzlibrary{decorations.pathreplacing}
\usetikzlibrary{shapes}
\usetikzlibrary{positioning, arrows, matrix, calc, patterns, fit}
\usepackage{caption}
\usepackage{array}
\usepackage{mdwmath}
\usepackage{multirow}
\usepackage{mdwtab}
\usepackage{eqparbox}
\usepackage{multicol}
\usepackage{amsfonts}
\usepackage{multirow,bigstrut,threeparttable}
\usepackage{amsthm}
\usepackage{array}
\usepackage{bbm}
\usepackage{epstopdf}
\usepackage{mdwmath}
\usepackage{mdwtab}
\usepackage{eqparbox}
\usepackage{tikz}
\usepackage{latexsym}
\usepackage{amssymb}
\usepackage{bm}
\usepackage{graphicx}
\usepackage{mathrsfs}
\usepackage{epsfig}
\usepackage{psfrag}
\usepackage{setspace}
\definecolor{linkColor}{HTML}{E74C3C}
\definecolor{pearcomp}{HTML}{B97E29}
\definecolor{citeColor}{HTML}{2980B9}
\definecolor{urlColor}{HTML}{1D2DEC}
\definecolor{conjColor}{HTML}{9ab569}
\usepackage[CJKbookmarks=true,
            bookmarksnumbered=true,
            bookmarksopen=true,
            colorlinks=true,
            citecolor=citeColor,
            linkcolor=linkColor,
            anchorcolor=red,
            urlcolor=urlColor,
            ]{hyperref}
\usepackage{natbib}
\setcitestyle{authoryear,round}
\newtheoremstyle{break}
  {\topsep}{\topsep}%
  {\itshape}{}%
  {\bfseries}{}%
  {\newline}{}%

\usepackage{pgfplots}
\usepackage{amsmath,amssymb}
\usepackage{mathtools}
\usepackage{stmaryrd}
\pgfmathdeclarefunction{gauss}{3}{%
  \pgfmathparse{#3/(#2*sqrt(2*pi))*exp(-((x-#1)^2)/(2*#2^2))}%
}
\pgfmathdeclarefunction{mingauss}{4}{%
  \pgfmathparse{#4/(#3*sqrt(2*pi))*exp(-max((x-#1)^2, (x-#2)^2)/(2*#3^2))}%
}

\tikzset{
  invisible/.style={opacity=0},
  visible on/.style={alt={#1{}{invisible}}},
  alt/.code args={<#1>#2#3}{%
    \alt<#1>{\pgfkeysalso{#2}}{\pgfkeysalso{#3}}
  },
}

\newtheorem{definition}{\textbf{Definition}}%

\newtheorem{lemma}{\textbf{Lemma}}%
\newtheorem{theorem}{\textbf{Theorem}}%

\newtheorem*{insight*}{\textbf{Observation}}
\newtheorem{proposition}{\textbf{Proposition}}[section]

\newtheorem{assumption}{Assumption}

\newtheorem*{lemmai*}{\textbf{Lemma (informal)}}
\newtheorem{remark}{\textbf{Remark}}

\usepackage[normalem]{ulem}
\usepackage[nameinlink]{cleveref}

\crefname{assumption}{assumption}{assumptions}
\crefname{table}{table}{tables}

\renewcommand{\cite}[1]{\citep{#1}}

\newcommand{\A}{\mathcal{A}}
\renewcommand{\S}{\mathcal{S}}
\newcommand{\Z}{\mathcal{Z}}
\newcommand{\D}{\mathcal{D}}
\newcommand{\F}{\mathcal{F}}
\newcommand{\N}{\mathcal{N}}
\newcommand{\C}{\mathcal{C}}
\newcommand{\E}{\mathbb{E}}
\newcommand{\W}{\mathcal{W}}
\newcommand{\Q}{\mathcal{Q}}
\newcommand{\V}{\mathcal{V}}
\newcommand{\M}{\mathcal{M}}
\newcommand{\T}{\mathsf{T}}
\newcommand{\Reg}{\text{Reg}}
\newcommand{\High}{\textsc{High}}
\newcommand{\Low}{\textsc{Low}}
\newcommand{\poly}{\text{poly}}

\newcommand{\vp}{\varepsilon/72\cdot\sqrt{\lambda\delta/|\Z|}}
\newcommand{\coverEps}{\varepsilon_0}
\newcommand{\cover}{\F,\coverEps}
\newcommand{\probln}{384L_1\cdot\ln(4\mathcal{N}(\cover)/\delta)}
\newcommand{\prob}{\probln/(\varepsilon^2\cdot|\Z|)}

\usepackage{color}
\definecolor{cm}{RGB}{0,0,200}
\definecolor{purple}{RGB}{200,0,200}

\usepackage[intoc]{nomencl}
\makenomenclature

\makeatletter
\newcommand{\vast}{\bBigg@{2.5}}
\newcommand{\Vast}{\bBigg@{5}}
\makeatother

\usepackage{bbm}

\def\poly{\operatorname{poly}}
\usepackage[utf8]{inputenc}
\usepackage[LGR,T1]{fontenc}

\usepackage{breqn}

\usepackage{enumitem,kantlipsum}

\title{Provably Efficient Reinforcement Learning via Surprise Bound}

\author{%
  Hanlin Zhu%
  \thanks{Department of Electrical Engineering and Computer Sciences, UC Berkeley. %
  	\url{hanlinzhu@berkeley.edu}}%
  \and %
  Ruosong Wang%
  \thanks{Paul G. Allen School of Computer Science \& Engineering, University of Washington. %
      \url{ruosongw@cs.washington.edu} }%
  \and 
  Jason D. Lee
   \thanks{ Electrical and Computer Engineering, Princeton University. %
      \url{jasonlee@princeton.edu} }%
}
\date{\today}

\begin{document}

\maketitle

\begin{abstract}
   Value function approximation is important in modern reinforcement learning (RL) problems especially when the state space is (infinitely) large. 
    Despite the importance and wide applicability of value function approximation, its theoretical understanding is still not as sophisticated as its empirical success, especially in the context of general function approximation.
   In this paper, we propose a  provably efficient  RL algorithm (both computationally and statistically) with general value function approximations. We show that if the value functions can be approximated by a function class $\mathcal{F}$ which satisfies the Bellman-completeness assumption, our algorithm achieves an $\widetilde{O}(\text{poly}(\iota H)\sqrt{T})$ regret bound where $\iota$ is the product of the surprise bound and log-covering numbers, $H$ is the planning horizon, $K$ is the number of episodes and $T = HK$ is the total number of steps the agent interacts with the environment. Our algorithm achieves reasonable regret bounds when applied to both the linear setting and the sparse high-dimensional linear setting. Moreover, our algorithm only needs to solve $O(H\log K)$ empirical risk minimization (ERM) problems, which is far more efficient than previous algorithms that need to solve ERM problems for $\Omega(HK)$ times.
\end{abstract}

\section{Introduction}

Modern Reinforcement Learning (RL) problems are often challenging due to the huge state spaces, and in practice, value function approximation schemes are usually employed to tackle this issue. 
Empirically, combining various reinforcement learning algorithms with function approximation schemes has led to tremendous success on various tasks~\citep{mnih2013playing,mnih2015human,silver2017mastering}.
However, despite the great empirical success, our theoretical understanding of RL with function approximation is still not as sophisticated as its empirical counterpart.
Until recently, most existing theoretical work in RL has been focusing on the tabular setting or the linear setting~\citep{azar2017minimax, jin2018q, yang2019sample, wang2019optimism, du2019provably, du2019good, agarwal2020model, wang2020long, du2020agnostic, jin2020provably, zanette2020learning, li2020breaking}, while in practice, complex function approximators like neural networks are usually employed. 
Over the years, understanding conditions on the function class that permit sample-efficient RL has evolved into an important open research problem in machine learning theory. 

Existing provably efficient RL algorithms that can handle general function approximation~\citep{jiang2017contextual, sun2019model, ayoub2020model, jin2021bellman, du2021bilinear} usually require solving computationally intractable optimization problems and are therefore computationally inefficient.
Recently,~\citet{wang2020reinforcement} proposed a provably efficient RL algorithm with general function approximation for function classes with bounded eluder dimensions. 
The algorithm by~\citet{wang2020reinforcement} is based on Least Squares Value Iteration (LSVI) and the principle of ``optimism in the face of uncertainty''. 
There are two shortcomings in the work of~\citet{wang2020reinforcement}. 
First, in order to calculate the exploration bonus, their algorithm applies sensitivity sampling~\citep{langberg2010universal, feldman2011unified, feldman2013turning} to reduce the size of the replay buffer.
Using a replay buffer with bounded complexity to calculate the exploration bonus is crucial for the correctness of their algorithm. 
On the other hand, such a step is complicated in nature and could be hard to implement in practice. 
Therefore, to make the algorithm practical, it is much more desirable to use simpler dimensionality reduction techniques (like uniform sampling) without sacrificing the theoretical guarantee. 
Second, as mentioned in~\citet{foster2018practical}, showing examples with a small eluder dimension beyond linearly parameterized functions is challenging. 
In addition, taking the worst-case over all histories, as in the definition of the eluder dimension, is usually overly pessimistic in practice. 
In contextual bandits, it is known that provable efficiency can be established by assuming distributional conditions on the problem.
For example,~\citet{foster2018practical} establishes regret bound for an optimism-based contextual bandits algorithm by assuming bounded surprise bound. 
It is natural to ask whether similar conditions can be used to establish provable efficiencies of RL algorithms. 

Recently,~\citet{foster2020instance} established instance-dependent regret bounds for contextual bandits and reinforcement learning problems by assuming a bounded disagreement coefficient, which is a distribution-dependent assumption. 
\citet{foster2020instance} show that the disagreement coefficient is always upper bounded by the eluder dimension of the function class. 
The RL algorithm in~\citet{foster2020instance}, which is also based on Least Squares Value Iteration (LSVI) and the principle of ``optimism in the face of uncertainty'', has two drawbacks.
First, their algorithm achieves provable guarantees only in the block MDP setting which might not be realistic in practice. 
Second, when calculating the exploration bonus, their algorithm uses the {\em star hull} to reduce the complexity of the replay buffer, which is also complicated in nature and therefore difficult to implement in practice.

In this paper, we develop a novel provably efficient RL algorithm with general function approximation.
Similar to previous algorithms~\citep{wang2020reinforcement, foster2020instance}, our algorithm is an optimistic version of LSVI. 
Compared to previous ones, our algorithm has the following advantages:
\begin{itemize}
\item The regret bound of our algorithm is based on a variant of surprise bound proposed in~\citep{foster2018practical}, which is a distribution-dependent quantity and could therefore be smaller than the eluder dimension which considers the worst-case over all histories. Moreover, our theory does not rely on the block MDP assumption.
Furthermore, the surprise bound can be upper bounded in the tabular setting, the linear setting and the high dimensional sparse linear setting, which implies our algorithm achieves reasonable regret bound in all these three settings. 

\item The dimensionality reduction technique for reducing the complexity of the replay buffer is based on uniform sampling.
This is much simpler than the sensitivity sampling framework in~\citet{wang2020reinforcement} and the method based on star hull in~\citet{foster2020instance}.
\item Our algorithm requires solving only $O(H \log K)$ empirical risk minimization (ERM) problems, while previous algorithms~\citep{wang2020reinforcement, foster2020instance} require solving $\Omega(HK)$ ERM problems. 
\end{itemize}

\subsection{Related work}
\paragraph{Tabular reinforcement learning.}
Tabular RL is well studied in the context of sample complexity and regret bound in numerous literature  \citep{kearns2002near,kakade2003sample,strehl2006pac,strehl2009reinforcement,jaksch2010near,azar2013minimax,lattimore2014near,dann2015sample,agrawal2017posterior,azar2017minimax,jin2018q,dann2019policy,zanette2019tighter,zhang2020almost,wang2020long,yang2021q}. In particular, for episodic MDP without further assumptions, the best regret bound is $\widetilde{O}(\sqrt{H^2SAT})$ for both model-based \citep{azar2017minimax} and model-free \citep{zhang2020almost} algorithms, which matches the lower bound $\Omega\left(\sqrt{H^2SAT}\right)$ proved by \citet{jin2018q}. Recently, \citet{yang2021q} propose an RL algorithm with a regret bound of $O\left(\frac{SA\text{poly}(H)}{\Delta_{\min}}\log(SAT)\right)$ assuming the existence of a positive sub-optimality gap. However, all algorithms mentioned above cannot be applied to RL problems with huge or infinite state spaces due to the polynomial dependence on $\sqrt{S}$ in the regret bound. Therefore, in this paper, we assume the value function lies in a function class with bounded complexity and design a provably efficient algorithm whose regret bound depends polynomially on the complexity of the function class instead of the size of the state space.

\paragraph{Bandits.}
There is also rich literature studying stochastic (contextual) bandits, which can be viewed as a special case of MDP without state transitions \citep{auer2002using,dani2008stochastic,li2010contextual,rusmevichientong2010linearly,chu2011contextual,abbasi2011improved,foster2018practical,foster2020instance,li2019nearly}. In particular, \citet{foster2018practical} study contextual bandit problems with general value function approximation, and prove their algorithms could achieve a regret bound depending polynomially on the surprise bound and the implicit exploration coefficient (IEC). In this paper, we study RL with general value function approximation, and prove that the regret bound of our algorithm also depends on the (slightly modified) surprise bound as well as the log-covering numbers. However, we note that the RL setting is much more complicated than the contextual bandits setting since there is no state transition in bandit problems.

\paragraph{Reinforcement learning with function approximation.}
In the setting of linear function approximation, there has been great interest recently in the theoretical analysis of the sample complexity of RL algorithms~\citep{yang2019sample,yang2020reinforcement,jin2020provably,cai2020provably,du2019provably,du2020agnostic,wang2019optimism,zanette2020learning,zhou2021nearly}.
Compared to linear function approximation, however, many current provably efficient algorithms for general value function approximation are relatively impractical. For example, algorithms in \citet{jiang2017contextual,sun2019model,dong2020root} achieve regret bound in terms of the witness rank or the Bellman rank, but they are not computationally efficient.
  \citet{foster2020instance} devise \textsc{RegRL} algorithm which is both computationally and statistically efficient. However, it requires the block MDP assumption which greatly alleviates the difficulty of (infinitely) huge state space and might not be realistic in practice.
\citet{ayoub2020model}  propose a model-based algorithm and \citet{wang2020reinforcement} propose a model-free algorithm for general value function approximation, and the regret bound of both algorithms depend on the eluder dimension. \citet{kong2021online} propose an efficient algorithm both computationally and statistically for general value function approximation, of which the regret bound also depends on the eluder dimension. However, the eluder dimension considers the worst-case over all histories and is thus often overly pessimistic. 
 Instead, the regret bound of our algorithm depends polynomially on the surprise bound which is a distribution-dependent quantity and thus could be smaller than the eluder dimension for practical scenarios. 
 
\section{Preliminaries}
\label{section_preliminaries}

 In this paper, we study episodic \emph{Markov Decision Process} (MDP) $\M = (\S, \A, H, P, r, \mu)$, where $\S$ is the state space, $\A$ is the finite action space, $H \in \mathbb{N}_+$ is the planning  horizon, $P: \S\times\A\ \to \Delta(\S)$ is the transition kernel which maps a state-action pair to a distribution over the state space, $r: \S \times \A \to [0,1]$ is the  reward function and $\mu \in \Delta(\S)$ is the initial state distribution~\footnote{Our analysis can be naturally extended to the time-inhomogeneous settings where the reward function and the transition kernel are different for each $h \in [H]$.}.

 A (stochastic) policy 
 \begin{align*}
     \pi = \{\pi_h \}_{h=1}^H : \S \times [H] \to \Delta(\A)
 \end{align*} maps any state $s$ to a distribution over the action space at each step $h$, where we use $[N]$ to denote the set $\{1,2,\ldots,N\}$ for any positive integer $N$. A trajectory
 \begin{align*}
     (s_1, a_1, r_1), (s_2, a_2, r_2), \ldots, (s_H, a_H, r_H)
 \end{align*} is induced by a policy $\pi$ if $s_1 \sim \mu$, $a_h \sim \pi_h(s_h), r_h = r(s_h, a_h), \forall h \in [H]$ and $s_{h+1} \sim P(s_h, a_h), \forall h \in [H-1]$. Furthermore, a policy $\pi = \{ \pi_h \}_{h=1}^H$ is deterministic if for each step $h \in [H]$, $\pi_h: \S \to \A$ maps a state to only one action.

 For any policy $\pi$, the expected cumulative reward starting from state $s$ at step $h$ is defined as the value function
\begin{align*}
    V_h^{\pi}(s) = \E_\pi\left[ \sum_{h'=h}^H r_{h'}  | s_h = s \right],
\end{align*} where we use superscript $\pi$  to denote that the trajectory is induced by $\pi$.
Similarly, the expected cumulative reward starting from state-action pair $(s,a)$ at step $h$ is defined as the $Q$-function
\begin{align*}
Q_h^{\pi}(s, a) = \E_\pi[ \sum_{h'=h}^H r_{h'}  | s_h = s, a_h = a ].
\end{align*}
Let $\pi^*$  denote the optimal policy which maximizes $\E_{s_1 \sim \mu } [V_1^{\pi}(s_1)]$. Also, let $V_h^*(s) = V_h^{\pi^*}(s)$ and $Q_h^*(s,a) = Q_h^{\pi^*}(s,a)$.

 The agent interacts with the environment for $K$ episodes. At the beginning of each episode $k \in [K]$, the agent specifies a policy $\pi^k$ based on previous trajectories and interacts with the environment using $\pi^k$ for $H$ steps. We assume the agent knows the number of episodes $K$, and we define $T = KH$ to be the total number of steps that the agent interacts with the environment. The \emph{regret} of an algorithm after $K$ episodes is defined as 
\begin{align*}
    \Reg(K) = \sum_{k=1}^K \left(V_1^*(s_1^k) - V_1^{\pi^k}(s_1^k)\right),
\end{align*}
which compares the accumulated rewards between the agent's policy and the optimal policy. The goal of the agent is to minimize the regret. In this paper, we consider the typical regime that $H$ is fixed while $K$ grows to infinity.

\paragraph{Width function and norms.} For notation convenience, we define the width function for any function class $\F \subseteq \{ f : \S\times\A\to\mathbb{R} \}$ and several norms for any function $f:\S\times\A\to\mathbb{R}$. 
The width function is defined as $$w(\F,s,a) = \max_{f,f'\in\F} \left(f(s,a)-f'(s,a)\right),$$ $\forall (s,a) \in \S\times\A$.
For any dataset $\Z  \subseteq \S\times\A$ and $\D  \subseteq \S\times\A\times\mathbb{R}$,  define $\Z$-norm $$\Vert f \Vert_{\Z} = \sqrt{\sum_{(s,a)\in\Z} f^2(s,a)},$$ $\D$-norm $$\Vert f \Vert_{\D} = \sqrt{\sum_{(s,a,r) \in \D} (f(s,a)-r)^2},$$ and infinite norm $$\Vert f\Vert_\infty = \max_{(s,a)\in\S\times\A} |f(s,a)|$$ respectively. In addition, define $\Vert v\Vert_\infty = \max_{s\in\S} |v(s)|$ for any  $v:\S\to\mathbb{R}$.

\paragraph{Additional notations for algorithms.} For any finite multiset $\mathcal{X}$, let $\texttt{Unif}(\mathcal{X})$ denote the uniform distribution over $\mathcal{X}$ and $\text{Card}_{\text{d}}(\mathcal{X})$ denote the number of distinct elements in $\mathcal{X}$. For any $x \in \mathbb{R}_+$, let $\lfloor x \rfloor$ denote the integer part of $x$ and define $\lceil x \rceil = \lfloor x \rfloor + 1$ if $x$ is not an integer and  otherwise $\lceil x \rceil = x$. We use the standard $O(\cdot), \Omega(\cdot)$ notations to hide constants and use $\Tilde{O}(\cdot), \Tilde{\Omega}(\cdot)$ to suppress log factors. Also, we use $x \lesssim y$ to denote that there exists a constant $c > 0$ s.t. $x \leq cy$, and use $x \gtrsim y$ if $y \lesssim x$.

\section{Algorithm}
\label{sec_alg}

In this section, we first introduce the assumptions for the algorithm and then present our main algorithm (\Cref{alg_main}). The theoretical guarantee of our algorithm is presented in \Cref{section_theoretical_guarantee}.

\subsection{Assumptions}
\label{subsec_assump}

 Assume our algorithm (\Cref{alg_main}) receives a function class $\F \subseteq \{f: \S\times\A \to [0,H+1] \}$ as part of the input. Since the complexity of $\F$ determines the efficiency of the algorithm, it is natural and necessary to require bounded complexities of the function class under appropriate measures.
 We make the following assumptions on the function class $\F$.

\begin{assumption}[Bellman-completeness]
\label{assump_bellman_backup}
    For any function $V: \S \to [0,H]$, there exists a function $f_V \in \F$, s.t.
    \begin{align*}
        f_V(\cdot,\cdot) = r(\cdot, \cdot) + \sum_{s' \in \S} P(s'|\cdot,\cdot)V(s').
    \end{align*}
\end{assumption}

\Cref{assump_bellman_backup} indicates the closedness under Bellman equations. 
This is a general assumption that summarizes many previous assumptions in special settings and is commonly adopted in previous literature for general value function approximation \citep{wang2020reinforcement,foster2020instance,kong2021online}. For tabular RL, $\F$ can be chosen as the set of all functions mapping from $\S\times\A$ to $[0,H+1]$. 
In the linear MDP setting~\citep{bradtke1996linear,jin2020provably,yang2019sample,yang2020reinforcement,wang2019optimism} where  the transition kernel and the reward function are both linear in a feature map $\phi: \S\times\A\to\mathbb{R}^d$, $\F$ can be the set of all linear functions with respect to $\phi$. In sparse high-dimensional linear MDP settings where the transition kernel and the reward function are both $s$-sparse linear functions in $\phi$, $\F$ can be the set of all $(2s)$-sparse linear functions with respect to $\phi$. Furthermore, \Cref{assump_bellman_backup} approximately holds in practice as long as $\F$ is rich enough (e.g., deep neural networks) and we show in \Cref{sec_misspecification} that  our algorithm is robust to model misspecification.

\begin{assumption}[Bounded covering number]
\label{assump_cover_number}
    Given any $\varepsilon > 0$, there exist covering sets $\C(\F, \varepsilon) \subseteq \F$ and $\C(\S\times\A, \varepsilon) \subseteq \S\times\A$ with bounded size $\N(\F, \varepsilon)$ and $\N(\S\times\A, \varepsilon)$ respectively, where 
    \begin{itemize}
        \item $\forall f \in \F$, $\exists f' \in \C(\F, \varepsilon),$ s.t. $\Vert f-f' \Vert_\infty \leq \varepsilon$.
        \item $\forall (s,a) \in \S\times\A$, 
        $\exists (s',a') \in \C(\S\times\A, \varepsilon),$ s.t. $\max_{f\in\F} |f(s,a)-f(s',a')| \leq \varepsilon$.
    \end{itemize}

\end{assumption}

\Cref{assump_cover_number} requires bounded covering numbers $\N(\cdot, \varepsilon)$ for both $\F$ and $\S\times\A$, and the regret bound of our algorithm depends only logarithmically on the covering numbers (\Cref{thm_main}). In the tabular RL setting, $\ln \N(\F,\varepsilon) = \widetilde{O}(|\S||\A|)$ and $\ln \N(\S\times\A,\varepsilon) = O(\ln(|\S||\A|)) $. In $d$-dimensional linear MDP settings, $\ln \N(\F,\varepsilon) = \widetilde{O}(d)$ and $\ln \N(\S\times\A,\varepsilon) = \widetilde{O}(d)$. In $s$-sparse  high-dimensional linear MDP settings, $\ln \N(\F,\varepsilon) = \widetilde{O}(s)$. If we further assume that $\phi(s,a)$ is $s$-sparse for all $(s,a)\in\S\times\A$, then $\ln \N(\S\times\A,\varepsilon) = \widetilde{O}(s)$. 

\paragraph{Surprise bound.} Another important complexity measure in this paper is \emph{surprise bound}, which was first introduced in \citet{foster2018practical} to characterize the complexity of the function class in the  contextual bandit setting.

\begin{definition}[Surprise bound]
\label{defn_vsb}
    The surprise bound is the smallest positive constant $L_1$ s.t.
    \begin{align*}
        (f(s,a) - f'(s,a))^2 
        \leq L_1 \E_{s' \sim \D_h(\pi) } \E_{a' \sim \pi_h(s')} \left[  (f(s', a') - f'(s', a'))^2\right]
    \end{align*}
    for all $f, f' \in \F, s \in \S, a \in \A, h \in [H]$ and any policy $\pi$, where $\D_h(\pi)$ is the distribution of $s_h$ when the policy is $\pi$.
\end{definition}

Intuitively, the surprise bound is small if all pairs of functions with a small expected squared error with respect to any policy, do not encounter a much larger squared error on any state-action pair. The following proposition gives upper bounds of the surprise bound for linear and sparse linear settings (see \Cref{sec_proof_of_surprise_bound_prop} for the proof).

\begin{proposition}
\label{prop_surprise_bound_special_case}
In the (sparse) linear MDP setting with a fixed feature map $\phi: \S\times\A\to \mathbb{R}^d$, consider the function class $\F = \{(s,a) \mapsto w^{\T}\phi(s,a) | w \in \W \}$ for some $\W \subseteq \mathbb{R}^d$.
\begin{itemize}
    \item If $\Vert \phi(s,a) \Vert_2 \leq 1, \forall (s,a)\in\S\times\A$ and $\Vert w \Vert_2 \leq 2H\sqrt{d}, \forall w\in\W$, then
    \begin{align*}
        L_1 \! \leq \! \! \! \sup_{\pi, h\in[H]}  \frac{1}{ \lambda_{\min}\left(\E_{s \sim \D_h(\pi), a \sim \pi_h(s) } \left[  \phi(s, a)\phi(s, a)^\T\right]\right)}.
    \end{align*}
    \item  If $\Vert \phi(s,a) \Vert_\infty \leq 1, \forall (s,a)\in\S\times\A$ and $\Vert w \Vert_\infty \leq 2H\sqrt{d}, \Vert w \Vert_0 \leq 2s, \forall w\in\W$, then
    \begin{align*}
        L_1  \! \! \leq \! \! \! \sup_{\pi, h \in[H]} \frac{4s}{ \psi_{\min}\left(\E_{s \sim \D_h(\pi), a \sim \pi_h(s) } \left[  \phi(s, a)\phi(s, a)^\T\right]\right)},
    \end{align*}
    where $\psi_{\min}(A) = \min_{w \neq 0: \Vert w\Vert_0 \leq 4s} w^{\T}Aw/ w^{\T}w$ is the minimum restricted eigenvalue for $(4s)$-sparse predictors~\citep{raskutti2010restricted}.
\end{itemize}
\end{proposition}

\subsection{Algorithm}
\label{subsec_alg}

In this section, we present our main algorithm (\Cref{alg_main}) and discuss in detail several important components of our algorithm.

\begin{algorithm}[htbp] 
\caption{Optimistic LSVI with doubling epoch schedule} 
\label{alg_main} 
\begin{algorithmic}[1] 
 \STATE \textbf{Input:} number of epochs $M$, number of warm-start \newline epochs $M_0$, failure probability $\delta \in (0, 1)$

\FOR{episode $k = 1, 2, \ldots, \tau_{M_0}-1$} 
\STATE Receive initial state $s_1^k \sim \mu$
\FOR{$h = 1, 2, \ldots, H$}
\STATE Take action $a^k_h \sim \texttt{Unif}(\A)$, observe $s^k_{h+1} \sim P(\cdot | s_h^k, a_h^k)$ and receive $r_h^k = r(s_h^k, a_h^k)$
\ENDFOR
\ENDFOR
\FOR{epoch $m = M_0, M_0+1, \ldots, M$}
\STATE $Q^{m}_{H+1}(\cdot , \cdot) \gets 0$ and $V_{H+1}^m(\cdot) \gets 0$
\STATE $\Z^m \gets \left\{ (s_h^k, a_h^k) \right\}_{(h,k)\in[H]\times[\tau_m-1]}$

\FOR {$h = H, H-1, \ldots, 1$}  
    \STATE $\D^m_h \gets \left\{ \left(s_{h'}^k, a_{h'}^k, r_{h'}^k +  V_{h+1}^{m}(s_{h'+1}^k)\right) \right\}$, $\forall (h',k) \in[H]\times[\tau_m]$
    
    \STATE $f_h^m \gets \arg \min_{f \in \F} \Vert f \Vert_{\D_h^m}^2$
    \STATE $b_h^m(\cdot,\cdot) \gets \texttt{Bonus}(\F, f_h^m, \Z^m, \delta)$ (\Cref{alg_bonus})
    
    \STATE $Q_h^m(\cdot,\cdot) \gets \min\left\{ f_h^m(\cdot,\cdot) + b_h^m(\cdot,\cdot), H \right\}$
    \STATE $V_h^m(\cdot) \gets \max_{a \in \A} Q_h^m(\cdot, a)$
    
    \STATE $\pi_h^m(\cdot) \gets \arg \max_{a\in \A} Q_h^m(\cdot, a)$
\ENDFOR

\FOR{episode $k = \tau_m, \tau_m+1, \ldots, \tau_{m+1}-1$}
\STATE Receive initial state $s_1^k \sim \mu$
\FOR{$h = 1, 2, \ldots, H$}
\STATE Take action $a^k_h \gets \pi_h^m(s_h^k)$, observe $s^k_{h+1} \sim P(\cdot | s_h^k, a_h^k)$ and receive $r_h^k = r(s_h^k, a_h^k)$
\ENDFOR
\ENDFOR
\ENDFOR 
\end{algorithmic}
\end{algorithm}

\subsubsection{Doubling epoch schedule} 
Our algorithm consists of $M$ epochs where each epoch $m \in [M]$ starts at the beginning of episode $\tau_{m} = 2^{m-1}$ and consists of $T_m = 2^{m-1}$ episodes. Thus, the total number of episodes $K = 2^M-1$ and $M = O(\log K)$. At the beginning of epoch $m$, the algorithm fixes a policy $\pi^m = \{ \pi^m_h \}_{h=1}^{H}$ and the agent executes $\pi^m$ for all episodes $k \in [\tau_m, \tau_m + T_m - 1]$. The $M$ epochs can be divided into two phases.

\begin{itemize}
\item \textbf{Phase 1: Warm-up epochs.} 
For the first $(M_0-1)$ epochs, the agent plays a uniformly random policy. These warm-up epochs are designed to encourage exploration at the initial episodes.

\item \textbf{Phase 2: Optimistic LSVI.} Starting from epoch $M_0$, we use an optimistic version of Least Squares Value Iteration (LSVI) similar to \citet{jin2020provably,wang2019optimism,wang2020reinforcement,foster2020instance}.
At the beginning of each epoch $m \geq M_0$, we maintain all previous trajectories as a replay buffer, and find the best fit $f^m = \{f^m_h\}_{h=1}^{H} \in \F^H$ with respect to the replay buffer in the sense of mean squared error (MSE), i.e., 
\begin{align*}
    f_h^m \gets \arg \min_{f \in \F} \Vert f \Vert_{\D_h^m}^2
\end{align*}
where $\D_h^m$ is the replay buffer (see definition in \Cref{alg_main}).
To avoid overfitting and encourage exploration, we design a bonus function $b_h^m(\cdot,\cdot)$ which we will discuss later in \Cref{sec_uniform_sampling}, and approximate the optimal $Q$ function $Q^*_h(\cdot, \cdot)$ by
\begin{align*}
    Q_h^m(\cdot,\cdot) = \min\left\{ f_h^m(\cdot,\cdot) + b_h^m(\cdot,\cdot), H \right\}.
\end{align*} 
Our design of the bonus function ensures that $Q_h^m$ is an optimistic estimator of $Q_h^*$ with high probability (\Cref{lem_optimistic_Q}). Finally, for each episode $k \in [\tau_m, \tau_{m+1}-1]$ in epoch $m$, the agent plays the greedy policy with respect to $Q_h^m$ and collect the trajectory in episode $k$.

\end{itemize}

The advantages of the doubling epoch schedule are two folded:

\begin{itemize}
\item \textbf{Computationally efficient.} Since our algorithm only conducts large amount of computation at the beginning of each epoch (computing $f^m_h$ by empirical risk minimization and $b^m_h$ by the width function as in \Cref{sec_uniform_sampling}, which can often be solved efficiently by appropriate optimization methods or assuming access to appropriate regression oracles~\citep{wang2020reinforcement,foster2018practical}) and there are only $O(\log K)$ epochs, our algorithm is much more computationally efficient than previous methods~\citep{wang2020reinforcement, foster2020instance} which require to solve $\Omega(HK)$ equivalent optimization problems. 

Recently, \citet{kong2021online} proposes an online sub-sampling technique which improves the computational complexity of \citet{wang2020reinforcement}. However, our algorithm is still much more computationally efficient than \citet{kong2021online}.
The algorithm of \citet{kong2021online} adopts sensitivity sampling, which requires computing sensitivities for each state action pair $(s_h^k, a_h^k)$. Since the calculation of sensitivity requires solving a regression oracle for $\Omega(\log (TH))$ times (see Section 4.4. in \citet{kong2021online}), and there are $T = KH$ such state-action pairs, their algorithm needs to solve $\Omega(KH \log (TH))$ regression oracles to calculate sensitivities and subsample the dataset.
While in our algorithm, we use uniform sampling to avoid the complex and time-consuming sensitivity calculation and thus does not need any oracle to perform the subsampling procedure.

\item \textbf{Stabilizing adjacent trajectories.} The doubling epoch schedule together with the warm-up epochs stabilizes the adjacent trajectories by ensuring that at the beginning of each epoch, at least half of the historical trajectories in the replay buffer are induced by the same policy. This property enables us to adopt uniform sampling (\Cref{alg_sampling}) to reduce the complexity of the replay buffer. 

\end{itemize}

\subsubsection{Uniform sampling}
\label{sec_uniform_sampling}

An important technical novelty of our algorithm is the design of the bonus function via uniform sampling. To ensure optimism of our estimator $Q_h^m$, we can choose $b_h^m$ as the upper bound of the difference between $Q_h^*$ and $f_h^m$. If we are able to obtain a confidence region $\F_h^m$ which contains both $f_h^m$ and $Q_h^*$, it suffices to define the bonus function as the width function of $\F_h^m$.  

A naive way to choose the confidence region is 
    $\F_h^m = \left\{ f \in \F \left| \Vert f - f_h^m \Vert_{\Z^m}^2 \leq \beta \right. \right\}$
with a carefully selected $\beta$. However, since the confidence region depends on the whole replay buffer with size at most $T$, the confidence region and thus the bonus function would suffer extremely high complexity. This implies that $\beta$ needs to be set extremely large to ensure the accuracy of the confidence region. To obtain a bonus function with low complexity, we reduce the complexity of the replay buffer by uniform sampling, which is formally stated in \Cref{alg_sampling}.

\begin{algorithm}
\caption{$\texttt{Uniform-Sampling}(\F, \Z, \lambda, \varepsilon, \delta)$}
\label{alg_sampling}
\begin{algorithmic}[1]
\STATE \textbf{Input:} function class $\F$, dataset $\Z$, parameters $\lambda, \varepsilon > 0$ and failure probability $\delta \in (0, 1)$

\STATE Set $\varepsilon_0 \gets \vp $
\STATE Set $p^{-1} \gets \max\left\{ 1, \left\lfloor \frac{1}{\prob} \right\rfloor \right\}$
\STATE Initialize $\Z' \gets \{\}$
\FOR{$z \in \Z$}
    \STATE Add $1/p$ copies of $z$ to $\Z'$ with probability $p$
\ENDFOR
\STATE \textbf{Output:} $\Z'$
\end{algorithmic}
\end{algorithm}

\paragraph{Comparison to previous methods.}
Actually, the algorithms in~\citet{wang2020reinforcement,foster2020instance} 
also suffer the high complexity of the bonus function and address the issue by sensitivity sampling and star hull respectively. However,  sensitivity sampling requires estimating the sensitivity of each state-action pair, which is time-consuming; the star hull is complicated in nature and thus is hard to implement in practice. In contrast, our uniform sampling is conceptually simple and easy to implement. 
 Note that there is only one single parameter $p$ to be determined in \Cref{alg_sampling}. When the surprise bound $L_1$ is known in advance, we can directly calculate the value of $p$. When $L_1$ is unknown, we can perform a grid-search in a log-space of $L_1$. Specifically, we can set a small value $L_{\min}$ as the lower bound of $L_1$ and a large value $L_{\max}$ as the upper bound, and perform \Cref{alg_main} for $L_1 \in \mathcal{L} \triangleq \{L_{\min}, 2L_{\min}, 2^2L_{\min}, \ldots, L_{\max} \}$. Then we can pick the policy with the best performance under different choices of $L_1$.
 
 \Cref{thm_main} shows that the regret of our main algorithm (\Cref{alg_main}) is $\Tilde{O}(\sqrt{T})$ in $T$ dependence. We also emphasize that the above grid-search procedure won't result in higher total regret, since one can first try each possible $L_1 \in \mathcal{L}$ for $O(\sqrt{T})$ times, and then exploit the best $L_1$ for the remaining $O(T - \sqrt{T}\log(L_{\max}/L_{\min})) = O(T)$ steps. The resulting total regret is still $\Tilde{O}(\sqrt{T})$.

\begin{algorithm}[htbp]
\caption{$\texttt{Bonus}(\F, \bar{f}, \Z, \delta)$}
\label{alg_bonus}
\begin{algorithmic}[1]
\STATE \textbf{Input:} function class $\F$, reference function $\bar{f}$, dataset $\Z$ and failure probability $\delta \in (0, 1)$

\STATE $\Z' \gets 
\texttt{Uniform-Sampling}(\F, \Z, \frac{\delta}{(16T)^2}, \frac{1}{2}, \frac{\delta}{16T} )$ (\Cref{alg_sampling})

\IF {$|\Z'| > 64T^2/\delta$  \OR $\text{Card}_{\text{d}}(\Z') \geq 9216L_1\cdot\ln(64T\mathcal{N}(\F, \delta/(9216T^2))/\delta)$}
\STATE $\Z' \gets \{\}$ 
\ENDIF
\STATE Let $\hat{f} \in \C(\F, 1/(8\sqrt{64T^2/\delta}))$ such that $\Vert \hat{f} - \bar{f} \Vert_{\infty} \leq 1/(8\sqrt{64T^2/\delta})$

\STATE $\widehat{\Z} \gets \{\}$
\FOR{$z \in \Z'$}
    \STATE Let $\hat{z} \in \C(\S\times\A, 1/(8\sqrt{64T^2/\delta}))$ such that  $\sup_{f\in\F} |f(z)-f(\hat{z})| \leq 1/(8\sqrt{64T^2/\delta})$
    \STATE $\widehat{\Z} \gets \widehat{\Z} \cup \{\hat{z}\}$
\ENDFOR

\STATE $\beta \triangleq \beta(\mathcal{F},\delta) \gets c' \cdot L_1H^2\ln^3(T/\delta)\ln(\N(\F,\delta/T^3))
 \times \ln(\N(\S\times\A,\delta/T^2))$ for some constant $c' > 0$
 
 \STATE $\widehat{\F} \gets \left\{ f \in \mathcal{F} \ | \   \Vert f-\hat{f}  \Vert_{\widehat{\Z}}^{2} \leq 3\beta + 2 \right\}$
\STATE \textbf{Output:} $\hat w(\cdot, \cdot) \gets w(\widehat{\F}, \cdot, \cdot)$

\end{algorithmic}
\end{algorithm}

\paragraph{Design of the bonus function via uniform sampling.} 
Now we are able to design a bonus function with low complexity as in \Cref{alg_bonus} via uniform sampling. After obtaining the reduced dataset $\Z'$, we round each data in $\Z'$ and the reference function $\bar{f}$ to their nearest neighbors in covering sets. The confidence region and the bonus function is then defined by the rounded reference function and the rounded dataset. Note that in \Cref{alg_bonus}, the rounding operation does not need to be performed explicitly since all the data are stored in computers with bounded precision, and thus all the data will be implicitly rounded. For the choice of $\beta$, we can use the same grid-search method of $L_1$ since $\beta$ is also determined by $L_1$.

\paragraph{Efficient computation of the bonus function.} The computation of the bonus function is equivalent to an optimization problem of the following form:
\begin{align*}
    &\max_{f_1, f_2 \in \mathcal{F}} f_1(s,a) - f_2(s,a) 
    \\ 
    & \quad \text{s.t.} \quad \|f_1 - f_2 \|_{\Z} \leq \varepsilon.
\end{align*}

This problem can be solved efficiently by either assuming access to an optimization oracle, or assuming access to only a regression oracle (which is a milder assumption than optimization oracles) as mentioned in Section 4.4 of \citet{kong2021online}.

\section{Theoretical results}
\label{section_theoretical_guarantee}

In this section, we formally present our main theorem of the regret bound and defer the proof to \Cref{appendix_proof_alg}.

\begin{theorem}[Main theorem]
\label{thm_main}
    Under \Cref{assump_bellman_backup}, \ref{assump_cover_number}, let $ M_0 = \left\lceil \ln\left( 16L_1^2\ln\frac{128T\N(\F, \delta/(9216T^2))^2} {\delta}\right)\right\rceil$ where the number of total steps $T = H\cdot (2^M-1)$ is sufficiently large. With probability at least $1-\delta$, the regret of $\Cref{alg_main}$ is at most 
    \begin{align*}
        O(\iota\cdot H^{3/2} \cdot \sqrt{T}),
    \end{align*}
    where 
    $
        \iota = L_1 \cdot \ln^2(T/\delta)\cdot\max(\ln(\N(\F,\delta/T^3)),\ln(\N(\S\times\A,\delta/T^2))).
    $
\end{theorem}

\begin{proof}[Proof sketch]

In this proof sketch, we ignore the rounding operation in \Cref{alg_bonus} for convenience. The proof can be decomposed into three main steps.

\begin{itemize}
    \item \textbf{Step 1: Bounding the complexity of the bonus function.}
First, we show that our bonus function has low complexity (\Cref{prop_stable_sampling}). 
 Note that the bonus function is defined as the width function of the confidence region 
 $$ \hat \F_h^m = \left\{ f \in \F \left| \Vert f - \hat f_h^m \Vert_{\hat Z^m}^2 \leq \beta \right. \right\}.$$
  Since the reduced dataset $\hat \Z^m$
has bounded size (\Cref{lem_size_bound}) and bounded number of distinct elements 
(\Cref{lem_num_distinct_elements}), our bonus function which is defined by $\hat \Z^m$ also has low complexity. Now it remains to show that the bonus function defined over the reduced dataset $\hat \Z^m$ is (almost) the same as the bonus function defined over the original dataset $\Z^m$. It is equivalent to show that the confidence region remains (almost) unchanged after uniform sampling. This can be proved by showing that for any function pairs $f, f' \in \F$, the $\Z'$-norm of $f-f'$ approximates well the $\Z$-norm of $f-f'$ (\Cref{lem::uniform_sampling_preserve_norm}).  For a  fixed function pair $(f, f')$, $\|f-f' \|_{\Z'}^2$ is an unbiased estimator of $\| f - f' \|_{\Z}^2$ and its variance can be controlled, since the trajectories in the replay buffer are stabilized by the doubling epoch and thus $\Z'$ has low complexity after uniform sampling. Then we can apply the Bernstein inequality to a fixed function pair $(f,f')$ to show that $\|f-f' \|_{\Z'}^2$ is close to $\|f-f' \|_{\Z}^2$ with high probability. Applying a union bound over all function pairs in the covering set of $\mathcal{F}$, we can obtain the desired result.

\item \textbf{Step 2: Optimism of the estimated $Q$-function.} 
The next step is to show that the estimated $Q$-function is an optimistic version of the true $Q$-function of the optimal policy (\Cref{lem_optimistic_Q}).
To achieve this, we need to show that the best fit $f_h^m $ is close to $r(\cdot,\cdot) + \sum_{s'\in\S} P(s'|\cdot,\cdot)V_{h+1}^m(s')$. If $f_h^m$ and  $V_{h+1}^m$ are independent, a standard concentration argument concludes the result. However, $V_{h+1}^m$ and $f_h^m$ are subtly dependent since they are both determined by the previous dataset. To address the difficulty, we first apply the standard concentration result on a fixed $V$ (\Cref{lem_single_step_error}), and then apply a union bound over all $V$ in a covering set (\Cref{lem_confidence_region}) to obtain the result. This method  is similar to \citet{wang2020reinforcement}.

\item \textbf{Step 3: Regret decomposition.} Finally, we decompose the regret by the summation of the bonus functions 
(\Cref{lem_regret_decomp}). Then, we use similar arguments as in   \citet{foster2018practical} to bound each bonus term by the surprise bound separately since the bonus function is defined as the (approximate) width function of the confidence region. 
\end{itemize}

\end{proof}

\begin{remark}
Recently,  \citet{foster2021statistical} proposes a high-level algorithm E2D. When applying E2D algorithm to our settings, one can show that it also achieves a similar regret bound $\Tilde{O}(\poly(L_1)\sqrt{T})$ (other parameters omitted).
However, we want to emphasize that E2D algorithm is too high-level to implement in practice. The implementation of E2D algorithm requires an online estimation oracle (see Algorithm 1 in \citet{foster2021statistical}), which is a very strong assumption in RL settings. While in our algorithm, we only require a ERM oracle and a regression oracle, which are mild and common assumptions in machine learning problems. 
\end{remark}

While our algorithm works for general value function class, it also achieves reasonable regret in special cases.

\paragraph{Tabular settings.}  In the tabular RL setting, it holds that $\ln \N(\F,\varepsilon) = \widetilde{O}(|\S||\A|)$ and $\ln \N(\S\times\A,\varepsilon) = O(\ln(|\S||\A|))$. When $\mu(s) \geq \varepsilon$ and $P(s'|s,a) \geq \varepsilon$ for all $s,s'\in\S$, $a\in\A$ for a (not too) small positive value $\varepsilon$,  $L_1 = O(\text{poly}(|\S||\A|))$, which implies that the regret bound is $\Tilde{O}(\text{poly}(|\S||\A|)H^{3/2}\sqrt{T})$. 
This is a reasonable regret bound since it is optimal in terms of $T$, the most important term in the regret bound, and has polynomial dependency in other parameters.

\paragraph{Linear settings.} When $\F$ is a $d$-dimensional linear function class, we have $\ln \N(\F,\varepsilon) = \ln \N(\S\times\A,\varepsilon) = \widetilde{O}(d)$. When $$\lambda_{\min}\left(\E_{s' \sim \D_h(\pi) } \E_{a' \sim \pi_h(s')} \left[  \phi(s', a')\phi(s', a')^\T\right]\right)$$ is lower bounded (of order $\Omega(1/d)$) and thus $L_1 = O(d)$ by \Cref{prop_surprise_bound_special_case}, the regret bound is $\widetilde{O}(d^2\cdot~ H^{3/2}~\cdot\sqrt{T})$, which is optimal in $T$-dependency and matches the result of \citet{wang2020reinforcement} in $d$-dependency.

\paragraph{Sparse linear settings.} Furthermore, when $\F$ is an $s$-sparse high-dimensional linear function class where typically $d \geq T \gg s$ , we have $\ln \N(\F,\varepsilon) = \widetilde{O}(s)$. When $$\psi_{\min}\left(\E_{s' \sim \D_h(\pi) } \E_{a' \sim \pi_h(s')} \left[  \phi(s', a')\phi(s', a')^\T\right]\right)$$ is lower bounded (of order $\Omega(1)$) and thus $L_1$ is $O(s)$ by \Cref{prop_surprise_bound_special_case}, the regret bound is $\widetilde{O}(s\cdot \max(s, \ln(\N(\S\times\A,\delta/T^2)))\cdot H^{3/2}\cdot\sqrt{T})$. If we further assume that $\phi(s',a')$ is $s$-sparse for all $(s',a')\in\S\times\A$, we have $\ln \N(\S\times\A,\varepsilon) = \widetilde{O}(s)$ and thus obtain an $\widetilde{O}(s^2\cdot H^{3/2}\cdot\sqrt{T})$ regret bound. However, directly applying the result in linear settings of \citet{wang2020reinforcement} can only obtain a linear regret when $d \geq T$.
This shows the superiority of our algorithm since we can provide theoretical guarantee for more general function classes, and thus it is an important step toward studying general value function approximation beyond the tabular and linear settings.

We also emphasize a subtle difference between linear and sparse linear settings. In linear settings, when $\lambda_{\min}\left(\E_{s' \sim \D_h(\pi) } \E_{a' \sim \pi_h(s')} \left[  \phi(s', a')\phi(s', a')^\T\right]\right)$ is lower bounded, we typically expect it to be of order $\Omega(1/d)$ since we assume the \emph{2-norm} $\| \phi \|_2 \leq 1$. While for sparse linear settings, when $\psi_{\min}\left(\E_{s' \sim \D_h(\pi) } \E_{a' \sim \pi_h(s')} \left[  \phi(s', a')\phi(s', a')^\T\right]\right)$ is lower bounded, we typically expect it to be of order $\Omega(1)$ since we assume the \emph{infinity norm} $\| \phi \|_\infty \leq 1$ in this setting.

\section{Model Misspecification}
\label{sec_misspecification}

Our main theorem (\Cref{thm_main}) requires Bellman-completeness assumption (\Cref{assump_bellman_backup}). Although the Bellman-completeness assumption is fairly common in theoretical analysis, especially in the presence of general value function approximation, the ground truth model together with the function class might slightly violate this assumption in real-world scenario. This phenomenon is known as model misspecification~\citep{jin2020provably,wang2020reinforcement}. 

 In this section, we show that as long as the violation of the Bellman-completeness assumption is small, the regret of our algorithm is still bounded. To state the result formally, we first introduce the following assumption, which can be viewed as a model misspecification version of the Bellman-completeness assumption.

\begin{assumption}[Model misspecification]
\label{assump_misspecification}
    There exists a constant $\zeta > 0$ satisfying that 
    for any function $V: \S \to [0,H]$, there exists a function $f_V \in \F$, s.t. 
    \begin{align*}
        \left\| f_V(\cdot,\cdot) - r(\cdot, \cdot) + \sum_{s' \in \S} P(s'|\cdot,\cdot)V(s') \right\|_\infty \leq \zeta.
    \end{align*}
\end{assumption}

Under \Cref{assump_misspecification}, one can directly apply \Cref{alg_main} to the model misspecification setting with only a different choice of the parameter $\beta$ in \Cref{alg_bonus}. Specifically, for some constant $c' > 0$ we set 
\begin{equation}
\label{eq::bonus_misspecification}
\begin{aligned}
    \beta = c' ( L_1H^2\ln^3(T/\delta)\ln(\N(\F,\delta/T^3))
  \ln(\N(\S\times\A,\delta/T^2)) + HT\zeta).
\end{aligned}
\end{equation}
Note that when \Cref{assump_bellman_backup} holds, it is equivalent to \Cref{assump_misspecification} with $\zeta = 0$, and thus the parameter $\beta$ is exactly the same as the one in our original algorithm.
The following theorem provides theoretical guarantees of our algorithm for model misspecification, and the proof is attached in \Cref{appendix_proof_of_misspecification}, which is very similar to the proof of \Cref{thm_main}.

\begin{theorem}[Theoretical guarantee for model misspecification]
\label{thm_misspecification}
    Under \Cref{assump_misspecification}, \ref{assump_cover_number}, let $ M_0 = \left\lceil \ln\left( 16L_1^2\ln\frac{128T\N(\F, \delta/(9216T^2))^2} {\delta}\right)\right\rceil$ and the number of total steps $T = H\cdot (2^M-1)$. With probability at least $1-\delta$, the regret of $\Cref{alg_main}$ (where the parameter $\beta$ is defined as in \eqref{eq::bonus_misspecification}) is at most 
    \begin{align*}
        O(\iota\cdot H^{3/2} \cdot \sqrt{T} + \sqrt{L_1 \cdot H^2 \cdot \zeta \cdot \log T}\cdot T),
    \end{align*}
    where 
    $
        \iota = L_1 \cdot \ln^2(T/\delta)\cdot\max(\ln(\N(\F,\delta/T^3)),\ln(\N(\S\times\A,\delta/T^2))).
    $
\end{theorem}

\section{Conclusion}

In this paper, we propose a  provably efficient  RL algorithm (both computationally and statistically) with general value function approximation. 
The regret bound of our algorithm depends on the surprise bound, which is a distribution-dependent quantity and could therefore be smaller than the eluder dimension considered in previous work. 
Our algorithm achieves reasonable regret bound when instantiating to special function classes. 

As a future direction, it would be interesting to see if it is possible to establish the provable efficiency of RL algorithms using other distribution-dependent complexity measures.
For example, it would be interesting to study whether it is possible to design a provably efficient RL algorithm by assuming a bounded disagreement coefficient (as in~\citet{foster2020instance}) but without the block MDP assumption. 

\bibliographystyle{plainnat}
\bibliography{reference}

\newpage
\appendix

\section{Analysis of the bonus function}

In this section, we analyze our bonus function, and the main proposition is presented in \Cref{prop_stable_sampling}.

\subsection{Analysis of \Cref{alg_sampling}}
\label{subsec_proof_sampling}

Note that the notation $\delta$ in \Cref{alg_bonus} and \Cref{alg_sampling} are different. In this subsection, all the notation $\delta$ refer to $\delta$ in \Cref{alg_sampling}, and therefore, $\lambda = \delta/(16T)$. Also, let $\varepsilon_0 = \vp$ throughout this subsection.

We assume that the input dataset of \Cref{alg_sampling} is $\Z = \{ (s_h^k, a_h^k) \}_{(h,k)\in[H]\times[t]}$ where more than half of the trajectories are induced by the same policy and the number of trajectories
    \begin{align*}
        t \geq 4L_1^2\ln\frac{8\N(\F,\varepsilon_0)^2}{\delta}
    \end{align*}
which is satisfied if $t \geq \tau_{M_0}$ and $M_0$ is chosen as in \Cref{thm_main}.

The first lemma gives an upper bound on the size of the dataset produced by uniform sampling.

\begin{lemma}
\label{lem_size_bound}
    With probability at least $1 - \delta/4$, $|\Z'| \leq 4|\Z|/\delta$.
\end{lemma}

\begin{proof}
    We define random variable 
     \begin{equation*}
        X_z=
        \begin{cases}
            1/p & \text{$z$ is added into $\Z'$ for $1/p$ times}\\
            0 & \text{otherwise}
        \end{cases}.
    \end{equation*}
    Since $|\Z'| = \sum_{z \in \Z} X_z$ and $\E[X_z] = 1$, we can obtain
    \begin{equation*}
        \Pr\{|\Z'| > 4|\Z|/\delta\} \leq \delta/4
    \end{equation*}
    by Markov inequality.
\end{proof}
The next lemma proves that after uniform sampling, the norms of difference of any function pairs are approximately preserved with high probability.

\begin{lemma}
\label{lem::uniform_sampling_preserve_norm}
     With probability at least $1-\delta/2$, for any $f, f' \in \mathcal{F}$,
    \begin{align*}
        (1-\varepsilon)\Vert f-f'\Vert_{\Z}^2 - 2\lambda \leq \Vert f-f'\Vert_{\Z'}^2 \leq (1+\varepsilon)\Vert f-f'\Vert_{\Z}^2 + 8|\Z|\lambda/\delta.
    \end{align*}
\end{lemma}

\begin{proof}
    When $p = 1$, $\Z = \Z'$, the result directly holds. So we only consider the case when $p < 1$, which means 
    \begin{align*}
        p \geq \prob.
    \end{align*}
    
    We separately consider the cases when $\Vert f-f' \Vert_{\Z}^2 < 2\lambda$ and $\Vert f-f' \Vert_{\Z}^2 \geq 2\lambda$.
    
    For any function pair $f, f' \in \F$ where $\Vert f-f' \Vert_{\Z}^2 < 2\lambda$, conditioned on the event in \Cref{lem_size_bound} which holds with probability at least $1 - \delta/4$, we can obtain that $\Vert f-f' \Vert_{\Z'}^2 \leq |\Z'| \Vert f-f' \Vert_{\Z}^2 \leq 4|\Z|/\delta  \cdot \Vert f-f' \Vert_{\Z}^2 \leq 8|\Z|\lambda/\delta.$ Also, by the fact that $\Vert f-f' \Vert_{\Z}^2 < 2\lambda$ and $\Vert f-f' \Vert_{\Z'}^2 \geq 0$, we can conclude that 
    \begin{equation*}
          (1-\varepsilon)\Vert f-f'\Vert_{\Z}^2 - 2\lambda \leq \Vert f-f'\Vert_{\Z'}^2 \leq (1+\varepsilon)\Vert f-f'\Vert_{\Z}^2 + 8|\Z|\lambda/\delta.
    \end{equation*}
    
    In the remaining part of the proof, we consider the case that $\Vert f-f' \Vert_{\Z}^2 \geq 2\lambda$.
    
    We first fix any pair of distinct functions $f, f' \in \mathcal{C}(\cover)$.  Assume the first $u = \lfloor (t+1)/2 \rfloor$ trajectories are all induced by the same policy $\pi$. Also, for any $1 \leq k \leq u$, let 
    \begin{align*}
        g_k =  \sum_{h=1}^H (f(s_h^k, a_h^k)-f'(s_h^k, a_h^k))^2.
    \end{align*}
    Therefore, 
    \begin{align*}
        \E\left[g_k\right] = \sum_{h=1}^H \mathbb{E}_{s \sim \mathcal{D}_h(\pi) } \mathbb{E}_{a \sim \pi_h(s)} \left[  (f(s, a) - f'(s, a))^2\right].   
    \end{align*}
    Note that 
    \begin{align*}
        0 \leq g_k \leq H\times \max_{(s,a)\in\S\times\A} (f(s,a)-f'(s,a))^2 =  H\Vert f-f'\Vert_\infty^2.   
    \end{align*}
    Also, by \Cref{defn_vsb}, \begin{align*}
        \E\left[g_k\right] \geq \frac{H}{L_1} \max_{s \in \S, a \in \A} (f(s,a)-f'(s,a))^2 = \frac{H}{L_1} \Vert f - f' \Vert_\infty^2.
    \end{align*}
    
    Therefore, by Hoeffding's inequality,
    \begin{align*}
        &\Pr \left\{ \frac{1}{u}\sum_{k=1}^u \left( g_k - \E\left[g_k\right] \right) \leq -v\E\left[g_1\right] \right\} \leq \exp\left( -\frac{2u^2v^2\E\left[g_1\right]^2}{uH^2\Vert f-f'\Vert_\infty^4}\right) \\
        \leq& \exp\left( -\frac{2uv^2}{H^2\Vert f-f'\Vert_\infty^4}\cdot\frac{H^2\Vert f-f'\Vert_\infty^4}{L_1^2}\right) \leq \exp\left( -\frac{tv^2}{L_1^2} \right) \\ 
        \leq& \exp\left(-\frac{v^2}{L_1^2} \cdot 4L_1^2\ln\frac{8\N(\F,\varepsilon_0)^2}{\delta}\right) \leq \exp\left( -4v^2\ln\frac{8\N(\F,\varepsilon_0)^2}{\delta}\right).
    \end{align*}
    Setting $v = \frac{1}{2}$, we can obtain 
    \begin{align*}
    \Pr \left\{ \frac{1}{u}\sum_{k=1}^u g_k  \leq \frac{1}{2}\E\left[g_1\right] \right\} \leq \frac{\delta}{8\N(\F,\varepsilon_0)^2}.
    \end{align*}
    
    Let $\mathcal{E}_1$ denote the event that 
    \begin{align*}
        \frac{1}{u}\sum_{k=1}^u g_k  \geq \frac{1}{2}\E\left[g_1\right],
    \end{align*}
    then $\Pr\{\mathcal{E}_1\} \geq 1 - \frac{\delta}{8\N(\F,\varepsilon_0)^2}$.
    
    Now, we condition on $\mathcal{E}_1$ for the following analysis. For each $z \in \Z$, define 
    \begin{equation*}
        X_z=
        \begin{cases}
            \frac{1}{p}(f(z) - f'(z))^2 & \text{$z$ is added into $\Z'$ for $1/p$ times}\\
            0 & \text{otherwise}
        \end{cases}.
    \end{equation*}
    
    Obviously, $\Vert f - f' \Vert_{\Z'}^2 = \sum_{z\in \Z} X_z,$ and $\E[X_z] = \left(f(z)-f'(z)\right)^2$. Also,
    \begin{align*}
        &\sum_{z \in \Z} \text{Var}[X_z] \leq \sum_{z \in \Z} \E[X_z^2] \leq \max_{z\in \Z}(f(z)-f'(z))^2/p \cdot \sum_{z \in \Z}(f(z)-f'(z))^2 \\
        = & \frac{\Vert f-f' \Vert_{\Z}^4}{p} \cdot \frac{\max_{z\in \Z}(f(z)-f'(z))^2}{\sum_{z \in \Z}(f(z)-f'(z))^2} \\ \leq& \frac{\Vert f-f' \Vert_{\Z}^4}{p} \cdot \frac{\frac{1}{H}\sum_{h=1}^{H}L_1 \mathbb{E}_{s \sim \mathcal{D}_h(\pi) } \mathbb{E}_{a \sim \pi_h(s)} \left[  (f(s, a) - f'(s, a))^2\right]}{\sum_{k=1}^{u}\sum_{h=1}^H (f(s_h^k, a_h^k) - f'(s_h^k, a_h^k))^2} \\
        \leq& \frac{\Vert f-f' \Vert_{\Z}^4}{p\cdot uH} \cdot \frac{L_1\E[g_1]}{\frac{1}{u}\sum_{k=1}^{u}g_k} \\
        \leq & \frac{2L_1\Vert f-f' \Vert_{\Z}^4}{pu\cdot H} \leq \frac{\Vert f-f' \Vert_{\Z}^4 \cdot \varepsilon^2}{96\cdot\ln(4\mathcal{N}(\cover)/\delta)}.
    \end{align*}
    Moreover,
    \begin{align*}
        \max_{z\in\Z} X_z =& \max_{z\in\Z} \frac{(f(z)-f'(z))^2}{p} \\
        \leq& \frac{\Vert f-f' \Vert_{\Z}^2}{p} \cdot \frac{\max_{z\in \Z}(f(z)-f'(z))^2}{\sum_{z \in \Z}(f(z)-f'(z))^2} \\
        =& \frac{\varepsilon^2 \Vert f-f' \Vert_{\Z}^2}{96\cdot\ln(4\mathcal{N}(\cover)/\delta)}
    \end{align*}
    Then, by Azuma-Bernstein's Inequality,
    \begin{align*}
        &\Pr\left\{ | \Vert f - f' \Vert_{\Z}^2 - \Vert f - f' \Vert_{\Z'}^2 | \geq \varepsilon/4\cdot \Vert f - f' \Vert_{\Z}^2 \left| \mathcal{E}_1\right. \right\} \\ =& 
        \Pr\left.\left\{ \left| \sum_{z \in \Z} \mathbb{E}[X_z] - \sum_{z \in \Z} X_z \right| \geq \varepsilon/4\cdot \Vert f - f' \Vert_{\Z}^2 \right| \mathcal{E}_1 \right\} \\ \leq& 2\exp{\left(-\frac{\varepsilon^2/16\cdot\Vert f-f'\Vert_{\Z}^4}{2\sum_{z\in\Z}\text{Var}[X_z] + 2/3\max_{z\in\Z}X_z \cdot \varepsilon/4 \cdot \Vert f-f' \Vert_{\Z}^2}\right)} \\ \leq&
        2\exp{\left(-\frac{\varepsilon^2/16\cdot\Vert f-f'\Vert_{\Z}^4\cdot \ln(4\mathcal{N}(\cover)/\delta)}{\Vert f-f'\Vert_{\Z}^4 \cdot \varepsilon^2/48 + \Vert f-f'\Vert_{\Z}^4 \cdot \varepsilon^2/576 }\right)} \\
        \leq&
        2\exp{\left(-2 \ln(4\mathcal{N}(\cover)/\delta)\right)} \\ 
        \leq& (\delta/8)/\left(\mathcal{N}(\cover)\right)^2.
    \end{align*}
    
    Since the above inequality holds conditioned on $\mathcal{E}_1$, if we do not condition on $\mathcal{E}_1$,  
    \begin{align*}
        \Pr\left\{ | \Vert f - f' \Vert_{\Z}^2 - \Vert f - f' \Vert_{\Z'}^2 | \geq \varepsilon/4\cdot \Vert f - f' \Vert_{\Z}^2\right\} \leq (\delta/4)/\left(\mathcal{N}(\cover)\right)^2.
    \end{align*}
    By union bound, the inequality above implies that with probability at least $1-\delta/4$, for any $f, f' \in \mathcal{C}(\cover)$, 
    \begin{align*}
        (1-\varepsilon/4)\Vert f-f'\Vert_{\Z}^2 \leq \Vert f-f'\Vert_{\Z'}^2 \leq (1+\varepsilon/4)\Vert f-f'\Vert_{\Z}^2.
    \end{align*}
    
    Denote the event above and the event in \Cref{lem_size_bound} by $\mathcal{E}_2$, where 
    \begin{align*}
        &\mathcal{E}_2 = \left\{ |\Z'| \leq 4|\Z|/\delta \right\}  \\ &\cap \left\{(1-\varepsilon/4)\Vert f-f'\Vert_{\Z}^2 \leq \Vert f-f'\Vert_{\Z'}^2 \leq (1+\varepsilon/4)\Vert f-f'\Vert_{\Z}^2, \forall f, f' \in \C(\F,\varepsilon_0)\right\}.        
    \end{align*}
    Now we condition on $\mathcal{E}_2$ where $\Pr\{\mathcal{E}_2\} \geq 1 - \delta/2$. For any function pair $f, f' \in \F$ where $\Vert f-f' \Vert_{\Z}^2 \geq 2\lambda$, there exists $\hat{f}, \hat{f'} \in \C(\F,\varepsilon_0)$, s.t.
    \begin{align*}
        \Vert f- \hat{f} \Vert_\infty \leq \varepsilon_0 = \vp \leq \sqrt{\lambda/(25|\Z|)}, \Vert f'- \hat{f'} \Vert_\infty \leq \sqrt{\lambda/(25|\Z|)}.
    \end{align*}
    Therefore,
    \begin{align*}
        (1-\varepsilon/4)\Vert \hat{f}-\hat{f'}\Vert_{\Z}^2 \leq \Vert \hat{f}-\hat{f'}\Vert_{\Z'}^2 \leq (1+\varepsilon/4)\Vert \hat{f}-\hat{f'}\Vert_{\Z}^2
    \end{align*}
    by $\mathcal{E}_2$. Then we can obtain that 
    \begin{align*}
        \Vert f - f' \Vert_{\Z'}^2 \leq& \left(\Vert f - \hat{f} \Vert_{\Z'} + \Vert \hat{f} - \hat{f'} \Vert_{\Z'} + \Vert \hat{f'} - f' \Vert_{\Z'} \right)^2 \\ 
        \leq& \left( (1+\varepsilon/8)\Vert \hat{f} - \hat{f'} \Vert_{\Z} + 2\sqrt{|\Z'|} \cdot \varepsilon_0 \right)^2 \\ 
        =& \left( (1+\varepsilon/8)\Vert \hat{f} - \hat{f'} \Vert_{\Z} + 2\sqrt{|\Z'|} \cdot \vp \right)^2 \\ 
        \overset{|\Z'| \leq 4|\Z|/\delta}{\leq}& \left( (1+\varepsilon/8)\Vert \hat{f} - \hat{f'} \Vert_{\Z} + \sqrt{\lambda}\cdot\varepsilon/18 \right)^2 \\
        \leq& \left( (1+\varepsilon/8)\Vert f - f' \Vert_{\Z} + \sqrt{\lambda}\cdot\varepsilon/18  + 2 \Vert \hat{f} - f \Vert_{\Z} + 2 \Vert \hat{f'} - f' \Vert_{\Z} \right)^2 \\ 
        \leq& \left( (1+\varepsilon/8)\Vert f - f' \Vert_{\Z} + \sqrt{\lambda}\cdot\varepsilon/18  + 4 \sqrt{|\Z|} \cdot \vp \right)^2 \\ 
        \leq& \left( (1+\varepsilon/8)\Vert f - f' \Vert_{\Z} + \sqrt{\lambda}\cdot\varepsilon/9\right)^2 \\ \overset{\Vert f - f' \Vert_{\Z} \geq \sqrt{\lambda}}{\leq} & (1+\varepsilon)\Vert f - f' \Vert_{\Z}^2.
    \end{align*}
    By similar methods, we can also obtain that 
    \begin{align*}
        \Vert f - f' \Vert_{\Z'}^2 \geq& \left(\Vert \hat{f} - \hat{f'} \Vert_{\Z'} - \Vert f - \hat{f} \Vert_{\Z'} - \Vert \hat{f'} - f' \Vert_{\Z'} \right)^2 \\ 
        \geq& \left( (1-\varepsilon/6)\Vert \hat{f} - \hat{f'} \Vert_{\Z} - 2\sqrt{|\Z'|} \cdot \vp \right)^2 \\ 
        \overset{|\Z'| \leq 4|\Z|/\delta}{\geq}& \left( (1-\varepsilon/6)\Vert \hat{f} - \hat{f'} \Vert_{\Z} - \sqrt{\lambda}\cdot\varepsilon/18 \right)^2 \\
        \geq& \left( (1-\varepsilon/6)\Vert f - f' \Vert_{\Z} - \sqrt{\lambda}\cdot\varepsilon/18  -  \Vert \hat{f} - f \Vert_{\Z} -  \Vert \hat{f'} - f' \Vert_{\Z} \right)^2 \\ 
        \geq& \left( (1-\varepsilon/6)\Vert f - f' \Vert_{\Z} - \sqrt{\lambda}\cdot\varepsilon/18  - 2 \sqrt{|\Z|} \cdot \vp \right)^2 \\ 
        \geq& \left( (1-\varepsilon/6)\Vert f - f' \Vert_{\Z} - \sqrt{\lambda}\cdot\varepsilon/12\right)^2 \\ \overset{\Vert f - f' \Vert_{\Z} \geq \sqrt{\lambda}}{\geq} & (1-\varepsilon)\Vert f - f' \Vert_{\Z}^2.
    \end{align*}
\end{proof}

We also give the bound of the number of distinct elements in $\Z'$.

\begin{lemma}
\label{lem_num_distinct_elements}
With probability at least $1-\delta/4$, $ \text{Card}_{\text{d}}(\Z') \leq 2304L_1\cdot\ln(4\mathcal{N}(\cover)/\delta) /\varepsilon^2.$
\end{lemma}
\begin{proof}
    First, note that 
    \begin{align*}
        p \leq 768L_1\cdot\ln(4\mathcal{N}(\cover)/\delta) /(\varepsilon^2\cdot|\Z|)
    \end{align*}
    since for any $0 < x < 1$, there must exists $\hat{x} \in [x, 2x]$ s.t. $1/\hat{x}$ is an integer.
    
    When $p = 1$, which means $\Z = \Z'$ and 
    \begin{align*}
        768L_1\cdot\ln(4\mathcal{N}(\cover)/\delta) /(\varepsilon^2\cdot|\Z|) \geq 1,    
    \end{align*}
    we have 
    \begin{align*}
        |\Z'| = |\Z| \leq 768L_1\cdot\ln(4\mathcal{N}(\cover)/\delta) /\varepsilon^2.
    \end{align*}
    
    When $p < 1$, we have $p \geq 384L_1\cdot\ln(4\mathcal{N}(\cover)/\delta) /(\varepsilon^2\cdot|\Z|)$.
    Now, For each $z \in \Z$, define 
    \begin{equation*}
        X_z=
        \begin{cases}
            1 & \text{$z$ is added into $\Z'$ for $1/p$ times}\\
            0 & \text{otherwise}
        \end{cases}.
    \end{equation*}
    Then the number of distinct elements in $\Z'$ is upper bounded by $\sum_{z\in\Z} X_z$. Since $\E[X_z] = p$, 
    \begin{align*}
        \sum_{z\in\Z} \E[X_z] = p\cdot|\Z| \leq 768L_1\cdot\ln(4\mathcal{N}(\cover)/\delta) /\varepsilon^2.
    \end{align*}
    By Chernoff bound,
    \begin{align*}
        & \Pr\left\{ \sum_{z\in\Z} X_z \geq 3\times768L_1\cdot\ln(4\mathcal{N}(\cover)/\delta) /\varepsilon^2 \right\} \leq \Pr\left\{ \sum_{z\in\Z} X_z \geq 3\sum_{z\in\Z} \E[X_z] \right\} \\ \leq& \exp\left\{-p\cdot|\Z|\right\} \leq \exp\left\{-384L_1\cdot\ln(4\mathcal{N}(\cover)/\delta) /\varepsilon^2\right\} \leq \exp\left\{-\ln(4/\delta)\right\} = \delta/4.
    \end{align*}
\end{proof}

\subsection{Analysis of \Cref{alg_bonus}}

In this subsection, all the notation $\delta$ refer to $\delta$ in \Cref{alg_bonus}. In other words, we replace all the $\delta$ in \Cref{subsec_proof_sampling} by $\delta/(16T)$. Also, we still assume that the input dataset of \Cref{alg_bonus} is $\Z = \{ (s_h^k, a_h^k) \}_{(h,k)\in[H]\times[t]}$ where more than half of the trajectories are induced by the same policy and the number of trajectories $t$ satisfies
    \begin{align*}
         4L_1^2\ln\frac{128T\N(\F, \delta/(9216T^2))^2}{\delta} \leq t \leq K = T/H,
    \end{align*}
which is satisfied if $t \geq \tau_{M_0}$ and $M_0$ is chosen as in \Cref{thm_main}.

Combining the three lemmas in \Cref{subsec_proof_sampling} with a union bound, we can obtain the following proposition.

\begin{proposition}
\label{prop_norm}
    Let $\Z'$ denote the dataset returned by \Cref{alg_sampling}. With probability at least $1-\delta/(16T)$, $|Z'| \leq 64T^2/\delta$, the number of distinct elements in $\Z'$ does not exceed 
    \begin{align*}
        9216L_1\cdot\ln(64T\mathcal{N}(\F, \delta/(9216T^2))/\delta),
    \end{align*}
    and for any $f, f' \in \F$, 
    \begin{equation*}
           \Vert f-f'\Vert_{\Z}^2/2 - 1/2 \leq \Vert f-f'\Vert_{\Z'}^2 \leq 3\Vert f-f'\Vert_{\Z}^2/2 + 1/2.
    \end{equation*}
\end{proposition}

By \Cref{prop_norm}, we can deduce the following proposition.

\begin{proposition}
\label{prop_stable_sampling}
    For \Cref{alg_bonus},
    the following holds.
    \begin{enumerate}
        \item With probability at least $1-\delta/(16T)$, 
        \begin{align*}
            w(\underline{\F}, s, a) \leq \hat{w}(s,a) \leq w(\overline{\F}, s, a),
        \end{align*}
        where $\underline{\F} = \left\{ f \in \F \ | \   \Vert f-\bar{f}  \Vert_{\Z}^{2} \leq \beta(\F, \delta) \right\}$ and $\overline{\F} = \left\{ f \in \F \ | \   \Vert f-\bar{f}  \Vert_{\Z}^{2} \leq 12\beta(\F, \delta) + 12 \right\}$. 
        \item There exists a function set $\W$ s.t. $\hat{w}(\cdot,\cdot) \in \W$ and
        \begin{align*}
            \ln |\W| \leq&      9216L_1\cdot\ln\left(64T\mathcal{N}(\F, \delta/(9216T^2))/\delta\right)\ln \left(\N(\S\times\A, 1/(8\sqrt{64T^2/\delta}))\times 64T^2/\delta\right)
            \\ &+ \ln\left(\N(\F, 1/(8\sqrt{64T^2/\delta}))\right) + 1
            \\ 
            \leq& C\cdot L_1\cdot\ln\left(\mathcal{N}(\F, \delta/T^3)\times T/\delta\right)\ln \left(\N(\S\times\A, \delta/T^2)\times T/\delta\right)
        \end{align*}
        for some absolute constant $C > 0$ when $T$ is sufficiently large.
    \end{enumerate}
\end{proposition}

\begin{proof}
    For the first part, we condition on the event defined in \Cref{prop_norm}. We only need to prove that $\underline{\F} \subseteq \widehat{\F} \subseteq \overline{\F}$, where $\widehat{\F}$ is defined in \Cref{alg_bonus}. For any $f \in \F$, we have 
    \begin{align*}
        \Vert f-\bar{f}\Vert_{\Z}^2/2 - 1/2 \leq \Vert f-\bar{f}\Vert_{\Z'}^2 \leq 3\Vert f-\bar{f}\Vert_{\Z}^2/2 + 1/2.
    \end{align*}
    Therefore,
    \begin{align*}
        \Vert f - \hat{f} \Vert_{\widehat{\Z}}^2 \leq& \left(\Vert f - \hat{f} \Vert_{\Z'} + \sqrt{64T^2/\delta}/(4\sqrt{64T^2/\delta}) \right)^2 \\
        \leq& \left(\Vert f - \bar{f} \Vert_{\Z'} + \sqrt{64T^2/\delta}/(8\sqrt{64T^2/\delta}) +  \sqrt{64T^2/\delta}/(4\sqrt{64T^2/\delta}) \right)^2 \\ 
        \leq& 2\Vert f - \bar{f} \Vert_{\Z'}^2 + 1/2 \leq 3\Vert f - \bar{f} \Vert_{\Z}^2 + 2. 
    \end{align*}
    This means for any $f \in \underline{\F}$, we have $\Vert f-\bar{f}  \Vert_{\Z}^{2} \leq \beta(\F, \delta)$, which implies $\Vert f - \hat{f} \Vert_{\widehat{\Z}}^2 \leq 3\beta(\F,\delta)+2$, i.e., $f \in \widehat{\F}$. Similarly,
    \begin{align*}
        \Vert f - \hat{f} \Vert_{\widehat{\Z}}^2 \geq& \left(\Vert f - \hat{f} \Vert_{\Z'} - \sqrt{64T^2/\delta}/(4\sqrt{64T^2/\delta}) \right)^2 \\
        \geq& \left(\Vert f - \bar{f} \Vert_{\Z'} - \sqrt{64T^2/\delta}/(8\sqrt{64T^2/\delta}) -  \sqrt{64T^2/\delta}/(4\sqrt{64T^2/\delta}) \right)^2 \\ 
        \geq& \Vert f - \bar{f} \Vert_{\Z'}^2/2 - 1/4 \geq \Vert f - \bar{f} \Vert_{\Z}^2/4 - 1. 
    \end{align*}
    So for any $f \in \widehat{\F}$, we have $\Vert f - \hat{f} \Vert_{\widehat{\Z}}^2 \leq 3\beta(\F,\delta)+2$, which implies $\Vert f-\bar{f}  \Vert_{\Z}^{2} \leq 12\beta(\F, \delta) + 12 $, i.e., $f \in \overline{\F}$.
    
    For the second part, since function $\hat{w}(\cdot,\cdot)$ is uniquely defined by $\widehat{\F}$,  we only need to analyze the maximal number of different possible function classes $\widehat{\F}$. When $|\Z '| > 64T^2/\delta$ or the number of distinct elements in $\Z '$ is larger than 
    \begin{align*}
        9216L_1\cdot\ln(64T\mathcal{N}(\F, \delta/(9216T^2))/\delta),
    \end{align*}
    $|\Z'| = 0$ and thus $\widehat{\F} = \F$. Otherwise, $\widehat{\F}$ is determined by $\widehat{\Z}$ and $\hat{f}$. Since $\hat{f} \in \C(\F, 1/(8\sqrt{64T^2/\delta}))$, the number of different $\hat{f}$ does not exceed $\N(\F, 1/(8\sqrt{64T^2/\delta}))$. Moreover, since there are at most 
    \begin{align*}
        9216L_1\cdot\ln(64T\mathcal{N}(\F, \delta/(9216T^2))/\delta)
    \end{align*}
    distinct elements in $\widehat{\Z}$, where $|\widehat{\Z}| \leq 64T^2/\delta$ and each element belongs to $\C(\S\times\A, 1/(8\sqrt{64T^2/\delta}))$, the number of different $\widehat{\Z}$ is upper bounded by
    \begin{align*}
        \left(\N(\S\times\A, 1/(8\sqrt{64T^2/\delta}))\times 64T^2/\delta\right)^{ 9216L_1\cdot\ln(64T\mathcal{N}(\F, \delta/(9216T^2))/\delta)}.
    \end{align*}
\end{proof}

\section{Analysis of the main algorithm}
\label{appendix_proof_alg}

Now we start to prove the regret bound of \Cref{alg_main}. The following lemma provides a bound on the estimation of a single backup.

\begin{lemma}[Single step optimization error] 
\label{lem_single_step_error}

Consider a fixed epoch $m \in [M]\backslash[M_0]$. We define 
\begin{align*}
    \Z^m = \left\{ (s_h^{k}, a_h^{k}) \right\}_{(h,k)\in[H]\times[\tau_m-1]}
\end{align*}
as in \Cref{alg_main}. Also, for any function $V: \S \to [0,H]$, we define
\begin{align*}
    \D_V^m = \left\{ \left(s_{h}^{k}, a_{h}^{k}, r_{h}^{k} + V(s_{h+1}^{k})\right) \right\}_{(h,k)\in[H]\times[\tau_m-1]}
\end{align*}
and
\begin{align*}
    \hat{f}_V = \arg\min_{f \in \F} \Vert f \Vert_{\D_V^m}^2.
\end{align*}
Then, for  any function $V: \S \to [0,H]$ and $\delta \in (0,1)$, there exists an event $\mathcal{E}_{V,\delta}$ where $\Pr\{\mathcal{E}_{V,\delta}\} \geq 1-\delta$, s.t. conditioned on $\mathcal{E}_{V,\delta}$, for any $V': \S \to [0,H]$ with $\Vert V-V' \Vert_\infty \leq 1/T$, we have 
\begin{align*}
    \left\Vert \hat{f}_{V'}(\cdot, \cdot) - r(\cdot, \cdot) - \sum_{s'\in\S} P(s'|\cdot,\cdot)V'(s') \right\Vert_{\Z^m}  \leq c'H\sqrt{\ln(T/\delta) +\ln \N(\F, 1/T)}.  
\end{align*}
for some constant $c' > 0$.
\end{lemma}

\begin{proof}
    For any $V: \S \to [0,H]$, we define 
    \begin{align*}
        f_V(\cdot, \cdot) = r(\cdot, \cdot) + \sum_{s'\in \S} P(s' | \cdot, \cdot)V(s'),
    \end{align*}
    and now we consider a fixed $V$. For any $f \in \F$, define 
    \begin{align*}
        \xi_h^{k}(f) = 2(f(s_h^k, a_h^k) - f_V(s_h^k, a_h^k) )\cdot(f_V(s_h^k, a_h^k) - r_h^k - V(s_{h+1}^k)), \forall (h, k) \in [H]\times[\tau_m-1].
    \end{align*}
    Also, for any $(h, k) \in [H]\times[\tau_m-1]$, define $\mathbb{F}_h^k$ as the filtration induced by
    \begin{align*}
        \{(s_{h'}^{k'}, a_{h'}^{k'}, r_{h'}^{k'})\}_{(h',k') \in [H]\times[k-1]} \cup \{ (s_{h'}^k, a_{h'}^k, r_{h'}^{k}) \}_{h' \in [h]}.
    \end{align*}
    Then we have $\mathbb{E}[\xi_h^k(f) | \mathbb{F}_h^k] = 0$ and $\mathbb{E}[(\xi_h^k(f))^2 | \mathbb{F}_h^k] \leq 4(H+1)^2(f(s_h^k,a_h^k) - f_V(s_h^k, a_h^k))^2$. Applying Lemma 10 of \citet{kirschner2018information} by setting $\{X_t\} = \{ \xi_h^k(f) \}$,  we can obtain that with probability at least $1-\delta$,
    \begin{align*}
        \sum_{(h,k)\in[H]\times[\tau_{m}-1]} \xi_h^k(f) \leq  8(H+1)^2\log\frac{2T+2}{\delta} +  4(H+1)\|f-f_V\|_{\mathcal{Z}^m} \sqrt{\log\frac{2T+2}{\delta}}.
    \end{align*}
    Applying a union bound of $\xi_h^k(f), -\xi_h^k(f)$ over all $f \in \mathcal{C}(\mathcal{F}, 1/T)$, we can further obtain that with probability at least $1-\delta$, 
    \begin{align*}
        &\left| \sum_{(h, k) \in [H]\times[\tau_m-1]} \xi_h^k(f) \right| \\ \leq& O\left(H^2(\ln (T/\delta)  + \ln \mathcal{N}(\mathcal{F}, 1/T)) +  H\|f-f_V\|_{\mathcal{Z}^m} \sqrt{\ln (T/\delta) + \ln \mathcal{N}(\mathcal{F}, 1/T)}\right)
    \end{align*}
    holds for all $f \in \mathcal{C}(\mathcal{F},1/T)$.

Let $\mathcal{E}_{V,\delta}$ denote the above event, and for the rest of the proof, we condition on $\mathcal{E}_{V,\delta}$.

Now, for any $f \in \F$, there exists a function $g \in \C(\F, 1/T)$, s.t. $\Vert f-g \Vert_\infty \leq 1/T$. Therefore,
\begin{align*}
        & \left| \sum_{(h, k) \in [H]\times[\tau_m-1]} \xi_h^k(f) \right| \leq \left| \sum_{(h, k) \in [H]\times[\tau_m-1]} \xi_h^k(g) \right| + 2(H+1)\Vert f-g\Vert_\infty|\Z^m| \\ 
        \lesssim& H^2(\ln (T/\delta)  + \ln \mathcal{N}(\mathcal{F}, 1/T)) +  H\|g-f_V\|_{\mathcal{Z}^m} \sqrt{\ln (T/\delta) + \ln \mathcal{N}(\mathcal{F}, 1/T)} \\ \lesssim& H^2(\ln (T/\delta)  + \ln \mathcal{N}(\mathcal{F}, 1/T)) +  H\|f-f_V\|_{\mathcal{Z}^m} \sqrt{\ln (T/\delta) + \ln \mathcal{N}(\mathcal{F}, 1/T)}. 
\end{align*}
For any $V': \S \to [0,H]$ with $\Vert V'-V \Vert_\infty \leq 1/T$, we can obtain that 
\begin{align*}
    \Vert f_{V'}-f_V \Vert_\infty  = \left\Vert\sum_{s' \in \S} P(s'|\cdot, \cdot)(V'(s')-V(s')) \right\Vert_\infty\leq \Vert V'-V \Vert_\infty \leq 1/T.
\end{align*}
Furthermore, for any $f \in \F$, 
\begin{align*}
    &\Vert f\Vert_{\D_{V'}^m}^2 - \Vert f_{V'}\Vert_{\D_{V'}^m}^2 - \Vert f - f_{V'} \Vert_{\Z^m}^2  \\ =&  2  \sum_{(s_h^k, a_h^k) \in \Z^m}(f(s_h^k, a_h^k) - f_{V'}(s_h^k, a_h^k) )\cdot(f_{V'}(s_h^k, a_h^k) - r_h^k - V'(s_{h+1}^k)) \\ \geq&  2  \sum_{(s_h^k, a_h^k) \in \Z^m}(f(s_h^k, a_h^k) - f_{V}(s_h^k, a_h^k) )\cdot(f_{V}(s_h^k, a_h^k) - r_h^k - V(s_{h+1}^k)) - 6(H+1) \\
    =&  \sum_{(h, k) \in [H]\times[\tau_m-1]} \xi_h^k(f)  - 6(H+1)\\
    \gtrsim& - H^2(\ln (T/\delta)  + \ln \mathcal{N}(\mathcal{F}, 1/T)) - H\|f-f_V\|_{\mathcal{Z}^m} \sqrt{\ln (T/\delta) + \ln \mathcal{N}(\mathcal{F}, 1/T)} \\
    \gtrsim& - H^2(\ln (T/\delta)  + \ln \mathcal{N}(\mathcal{F}, 1/T)) - H\|f-f_{V'}\|_{\mathcal{Z}^m} \sqrt{\ln (T/\delta) + \ln \mathcal{N}(\mathcal{F}, 1/T)}.
\end{align*}
If we let $f = \hat{f}_{V'}$, since $\hat{f}_{V'} = \arg\min_{f\in\F} \Vert f\Vert_{\D_{V'}^m}$, we have
\begin{align*}
    0 \geq& \Vert \hat{f}_{V'}\Vert_{\D_{V'}^m}^2 - \Vert f_{V'}\Vert_{\D_{V'}^m}^2 \\ \gtrsim& \Vert \hat{f}_{V'} - f_{V'} \Vert_{\Z^m}^2  - H^2(\ln (T/\delta)  + \ln \mathcal{N}(\mathcal{F}, 1/T)) - H\|\hat{f}_{V'}-f_{V'}\|_{\mathcal{Z}^m} \sqrt{\ln (T/\delta) + \ln \mathcal{N}(\mathcal{F}, 1/T)},
\end{align*}
which implies
\begin{align*}
    \Vert \hat{f}_{V'} - f_{V'} \Vert_{\Z^m} \leq c'H\sqrt{\ln(T/\delta)+\ln\N(\F,1/T)}.
\end{align*}
for some constant $c' > 0$.
\end{proof}

\begin{lemma}[Confidence region]
\label{lem_confidence_region}
In \Cref{alg_main}, for $m > M_0$, define confidence region
\begin{align*}
    \F_h^m = \left\{ f \in \F \left| \Vert f - f_h^m \Vert_{\Z^m}^2 \leq \beta(\F, \delta) \right. \right\}.
\end{align*}
Then with probability at least $1-\delta/16$, for all $(h, m) \in [H]\times([M]\backslash[M_0])$,
\begin{align*}
   r(\cdot, \cdot) + \sum_{s'\in \S} P(s' | \cdot, \cdot)V_{h+1}^m(s') \in \F_h^m,
\end{align*}
given
\begin{align*}
    \beta(\F, \delta) \geq c'H^2(\ln(T/\delta) +\ln \N(\F, 1/T) + \ln |\W|).
\end{align*}
for some constant $c' > 0$. Here, $\W$ is given in \Cref{prop_stable_sampling}.
\end{lemma}

\begin{proof}
    By \Cref{prop_stable_sampling}, $b_h^m(\cdot, \cdot) \in \W, \forall (h, m) \in [H]\times([M]\backslash[M_0]).$ Note that 
    \begin{align*}
        \Q = \{\min\{f(\cdot,\cdot)+w(\cdot,\cdot), H\} |f \in \C(\F,1/T),  w \in \W\} \cup \{0\}
    \end{align*}
    is a $(1/T)$-cover of 
    \begin{equation*}
        Q_{h+1}^m(\cdot,\cdot)=
        \begin{cases}
            \min\{f_{h+1}^m(\cdot, \cdot) + b_{h+1}^m(\cdot, \cdot), H\}, & h < H\\
            0, & h = H
        \end{cases},
    \end{equation*}
    i.e., there exists $q \in \Q$, s.t. $\Vert q - Q_{h+1}^m\Vert_\infty \leq 1/T$. Therefore, 
    \begin{align*}
        \V = \left\{ \max_{a \in A} q(\cdot, a) | q \in \Q \right\}
    \end{align*}
    is a $(1/T)$-cover of $V_{h+1}^m$ with $\ln|\V| \leq \ln|\W| + \ln\N(\F, 1/T) + 1$.
    
    Now, for each $V \in \V$, let $\mathcal{E}_{V, \delta/(16|\V|T)}$ denote the event defined in \Cref{lem_single_step_error}. By union bound, $\Pr\{\bigcap_{V \in \V} \mathcal{E}_{V, \delta/(16|\V|T)}\} \geq 1 - \delta/(16T).$ In the rest of the proof, we condition on the event $\bigcap_{V \in \V} \mathcal{E}_{V, \delta/(16|\V|T)}$.
    
    Since $f_h^m = \arg \min_{f \in \F} \Vert f \Vert_{\D_h^m}^2$, and there exists $V \in \V$ s.t. $\Vert V - V_{h+1}^m \Vert_\infty \leq 1/T$,  by \Cref{lem_single_step_error}, we have 
    \begin{align*}
         \left\Vert f_h^m(\cdot, \cdot) - r(\cdot, \cdot) - \sum_{s'\in\S} P(s'|\cdot,\cdot)V_{h+1}^m(s') \right\Vert_{\Z^m}  \leq c'H\sqrt{\ln(T/\delta) +\ln \N(\F, 1/T) + \ln |\W|}
    \end{align*}
    for some constant $c' > 0$. Applying a union bound over all $(h, m) \in [H]\times([M]\backslash[M_0])$, we have that with probability at least $1-\delta/16$,
    \begin{align*}
        r(\cdot, \cdot) + \sum_{s'\in \S} P(s' | \cdot, \cdot)V_{h+1}^m(s') \in \F_h^m, \forall (h, m) \in [H]\times([M]\backslash[M_0]).
    \end{align*}
\end{proof}

The above lemma proves that the confidence region contains $r(\cdot, \cdot) + \sum_{s'\in \S} P(s' | \cdot, \cdot)V_{h+1}^m(s')$ with high probability, which implies that all the estimated $Q$-function $Q_h^m$ are optimistic with high probability as well. We formally state the conclusion in the next lemma.

\begin{lemma}[Optimistic $Q$-function]
\label{lem_optimistic_Q}
With probability at least $1-\delta/8$, 
\begin{align*}
    Q_h^*(s,a) \leq Q_h^m(s,a) \leq r(s,a) + \sum_{s'\in\S} P(s'|s,a)V_{h+1}^m(s') + 2b_h^m(s,a)
\end{align*}
for all $(h, m) \in [H]\times([M]\backslash[M_0])$  and $(s,a) \in \S\times\A$.
\end{lemma}

\begin{proof}
    Let $\F_h^m$ be the confidence region as defined in \Cref{lem_confidence_region}. Let $\mathcal{E}_1$ denote the event that 
     \begin{align*}
        r(\cdot, \cdot) + \sum_{s'\in \S} P(s' | \cdot, \cdot)V_{h+1}^m(s') \in \F_h^m, \forall (h, m) \in [H]\times([M]\backslash[M_0]).
    \end{align*}
    By \Cref{lem_confidence_region}, $\Pr\{\mathcal{E}_1\} \geq 1-\delta/16$.
    Let $\mathcal{E}_2$ denote the event that 
    \begin{align*}
        b_h^m(s,a) \geq w(\F_h^m, s, a), \forall (h, m) \in [H]\times([M]\backslash[M_0]), (s, a) \in \S \times \A.
    \end{align*}
    By \Cref{prop_stable_sampling} and union bound over all $(h, m) \in [H]\times([M]\backslash[M_0])$, $\Pr\{\mathcal{E}_2\} \geq 1-\delta/16$. We condition on $\mathcal{E}_1\cap\mathcal{E}_2$ in the rest of the proof, which holds with failure probability at most $\delta/8$.
    
    By the definition of width function,
    \begin{align*}
        \max_{f \in \F_h^m} |f(s,a) - f_h^m(s,a)| \leq w(\F_h^m, s,a) \leq b_h^m(s, a), \forall (s,a)\in\S\times\A.
    \end{align*}
    Since $r(\cdot, \cdot) + \sum_{s'\in \S} P(s' | \cdot, \cdot)V_{h+1}^m(s') \in \F_h^m$, we have 
    \begin{equation}
    \label{eq::bonus_triangle_inequality}
         \left|r(s,a) + \sum_{s'\in \S} P(s' | s, a)V_{h+1}^m(s') - f_h^m(s,a)\right| \leq b_h^m(s, a), \forall (s,a)\in\S\times\A.
    \end{equation}
    Therefore, for all $ (s,a)\in\S\times\A,$
    \begin{align*}
        Q_h^m(s,a) \leq f_h^m(s,a)+b_h^m(s,a) \leq r(s,a) + \sum_{s'\in \S} P(s' | s, a)V_{h+1}^m(s') + 2b_h^m(s, a).
    \end{align*}
    
    Next, we start to prove $Q_h^*(\cdot,\cdot) \leq Q_h^m(\cdot,\cdot)$ by induction on $h$. When $h = H+1$, the inequality directly holds since  $Q_{H+1}^*(\cdot,\cdot) = Q_{H+1}^m(\cdot,\cdot) = 0$. Now for any $h \in [H]$, assume $Q_{h+1}^*(\cdot,\cdot) \leq Q_{h+1}^m(\cdot,\cdot)$. This also implies $V_{h+1}^*(\cdot) \leq V_{h+1}^m(\cdot)$. Therefore, for any $(s,a) \in \S\times\A$, 
    \begin{align*}
        Q_h^*(s,a) =& r(s,a) + \sum_{s'\in\S} P(s'|s,a)V_{h+1}^*(s') \\
        \leq& \min\left\{ H, r(s,a) + \sum_{s'\in\S} P(s'|s,a)V_{h+1}^m(s') \right\} \\
        \overset{\eqref{eq::bonus_triangle_inequality}}{\leq}& \min\left\{ H, f_h^m(s,a)+b_h^m(s,a) \right\} = Q_h^m(s,a),
    \end{align*}
    which completes the proof.
\end{proof}

Now, we can decompose the regret and bound it by the summation of bonus functions.

\begin{lemma}[Regret decomposition]
\label{lem_regret_decomp}
With probability at least $1-\delta/4$, 
\begin{align*}
    \Reg (K) \leq \tau_{M_0+1}\cdot H + 2\sum_{m = M_0+1}^M \sum_{k=\tau_m}^{\tau_{m+1}-1} \sum_{h=1}^H b_h^m(s_h^k, a_h^k) + 8H\sqrt{T\ln(16/\delta)}
\end{align*}
\end{lemma}

\begin{proof}
    For any step $h \in [H]$, epoch $m \in [M]\backslash[M_0]$ and episode $k$ in epoch $m$, define \begin{align*}
        \xi_h^k = \sum_{s' \in \S} P(s'|s_h^k, a_h^k) \left( V_{h+1}^m(s') - V_{h+1}^{\pi^m}(s')\right) - \left( V_{h+1}^m(s_{h+1}^k) - V_{h+1}^{\pi^m}(s_{h+1}^k)\right),
    \end{align*}
    and define $\mathbb{F}_h^k$ as the filtration induced by 
    \begin{align*}
        \{(s_{h'}^{k'}, a_{h'}^{k'}, r_{h'}^{k'})\}_{(h',k') \in [H]\times[k-1]} \cup \{ (s_{h'}^k, a_{h'}^k, r_{h'}^{k}) \}_{h' \in [h-1]}.
    \end{align*}
    Then $\E[\xi_h^k|\mathbb{F}_h^k]= 0$ and $|\xi_h^k| \leq 2H$. By Azuma-Hoeffding inequality, with probability at least $1-\delta/8$, 
    \begin{align*}
        \sum_{m=M_0+1}^M \sum_{k=\tau_m}^{\tau_{m+1}-1} \sum_{h=1}^{H}  \xi_h^k \leq 8H\sqrt{T\ln(16/\delta)}.
    \end{align*}
    We condition on both this event and the event defined in \Cref{lem_optimistic_Q} which also holds with probability at least $1-\delta/8$ in the rest of the proof.  
    
    Let $\pi^0$ denote the uniformly random policy adopted in the first $(M_0-1)$ epochs. By \Cref{lem_optimistic_Q}, 
    \begin{align*}
        \Reg(K) =& \sum_{k=1}^{\tau_{M_0-1}} \left( V_1^*(s_1^k) - V_1^{\pi^0}(s_1^k)\right) +
        \sum_{m = M_0}^M \sum_{k=\tau_m}^{\tau_{m+1}-1} \left( V_1^*(s_1^k) - V_1^{\pi^m}(s_1^k)\right) \\
        \leq& \tau_{M_0+1}\cdot H + \sum_{m = M_0+1}^M \sum_{k=\tau_m}^{\tau_{m+1}-1} \left( V_1^m(s_1^k) - V_1^{\pi^m}(s_1^k)\right).
    \end{align*}
    
    For each $k$ and corresponding $m$, we have
    \begin{align*}
        &V_1^m(s_1^k) - V_1^{\pi^m}(s_1^k)  \\
         =& Q_1^m(s_1^k, a_1^k)  - r(s_1^k,a_1^k) - \sum_{s'\in\S} P(s'|s_1^k,a_1^k)V_2^{\pi^m}(s')\\ 
         \leq& r(s_1^k,a_1^k) + \sum_{s'\in\S} P(s'|s_1^k,a_1^k)V_2^m(s') + 2b_1^m(s_1^k,a_1^k) - r(s_1^k,a_1^k) - \sum_{s'\in\S} P(s'|s_1^k,a_1^k)V_2^{\pi^m}(s') \\ 
         =&  \sum_{s'\in\S} P(s'|s_1^k,a_1^k)(V_2^m(s')-V_2^{\pi^m}(s')) + 2b_1^m(s_1^k,a_1^k)  \\ 
         =& (V_2^m(s_2^k)-V_2^{\pi^m}(s_2^k)) + \xi_1^k + 2b_1^m(s_1^k,a_1^k) \\ 
         \leq&  (V_3^m(s_3^k)-V_3^{\pi^m}(s_3^k)) + \xi_1^k + \xi_2^k + 2b_1^m(s_1^k,a_1^k) + 2b_2^m(s_2^k,a_2^k)  \\
         \leq& \cdots \\
         \leq& \sum_{h=1}^H \left(\xi_h^k + 2b_h^m(s_h^k, a_h^k)\right). 
    \end{align*}
    Therefore, 
    \begin{align*}
        \Reg(K) \leq \tau_{M_0+1}\cdot H + 2\sum_{m = M_0+1}^M \sum_{k=\tau_m}^{\tau_{m+1}-1} \sum_{h=1}^H b_h^m(s_h^k, a_h^k) + 8H\sqrt{T\ln(16/\delta)}.
    \end{align*}
\end{proof}

To prove the main theorem, we also need the next lemma. 

\begin{lemma}
\label{lem_bound_of_square_error}
With probability at least $1-\delta/2$, for all $(h,m) \in [H]\times([M]\backslash[M_0])$ and any $f,f' \in \F$, 
\begin{align*}
    T_{m-1}\E_{s\sim \D_h(\pi^{m-1}), a\sim\pi_h^{m-1}(s)} [(f(s,a) - f'(s,a))^2 ]\leq 4\sum_{k=\tau_{m-1}}^{\tau_m-1} (f(s_h^k,a_h^k)-f'(s_h^k,a_h^k))^2 + 64.
\end{align*}
\end{lemma}
\begin{proof}
    We first fix any $(h,m) \in [H]\times([M]\backslash[M_0])$. Define dataset 
    \begin{align*}
        \Z_h^m = \left\{ (s_h^k, a_h^k) \right\}_{k\in[\tau_{m-1}, \tau_m-1]}.
    \end{align*}
    Now we fix any pair of distinct functions $f, f' \in \C(\F, 1/T)$. Also, for any episode $k \in [\tau_{m-1}, ~\tau_m~-~1~]~$, let 
    \begin{align*}
        \xi_h^k =  (f(s_h^k, a_h^k)-f'(s_h^k, a_h^k))^2.
    \end{align*}
    Therefore, 
    \begin{align*}
        \E\left[\xi_h^k\right] = \E_{s\sim \D_h(\pi^{m-1}), a\sim\pi_h^{m-1}(s)} \left[  (f(s, a)-f'(s, a))^2\right].   
    \end{align*}
    Note that 
    \begin{align*}
        0 \leq \xi_h^k \leq  \max_{(s,a)\in\S\times\A} (f(s,a)-f'(s,a))^2 =  \Vert f-f'\Vert_\infty^2.
    \end{align*}
    Also, by \Cref{defn_vsb}, \begin{align*}
        \E\left[\xi_h^k\right] \geq \frac{1}{L_1} \max_{s \in \S, a \in \A} (f(s,a)-f'(s,a))^2 = \frac{1}{L_1} \Vert f - f' \Vert_\infty^2.
    \end{align*}
    
    Therefore, by Hoeffding's inequality,
    \begin{align*}
        &\Pr \left\{ \frac{1}{T_{m-1}}\sum_{k=\tau_{m-1}}^{\tau_m-1} \left(
        \xi_h^k - \E\left[\xi_h^k\right] \right) \leq -v\E\left[\xi_h^{\tau_{m-1}}\right] \right\} \leq \exp\left( -\frac{2T_{m-1}^2v^2\E\left[\xi_h^{\tau_{m-1}}\right]^2}{T_{m-1}\Vert f-f'\Vert_\infty^4}\right) \\
        \leq& \exp\left( -\frac{2T_{m-1}v^2}{\Vert f-f'\Vert_\infty^4}\cdot\frac{\Vert f-f'\Vert_\infty^4}{L_1^2}\right) \leq \exp\left( -\frac{2T_{m-1}v^2}{L_1^2} \right).
    \end{align*}
    Since
    \begin{align*}
        T_{m-1} = 2^{m-2} \geq 2^{M_0-1} \geq 8L_1^2\ln\frac{128T\N(\F, \delta/(9216T^2))^2}{\delta} \geq 2L_1^2\ln\frac{2T\N(\F, 1/T)^2}{\delta},
    \end{align*}
    by setting $v = \frac{1}{2}$, we can obtain that  
    \begin{align*}
    &\Pr \left\{ \frac{1}{T_{m-1}}\sum_{k=1}^u \xi_h^k  \leq  \frac{1}{2}\E\left[\xi_h^{\tau_{m-1}}\right] \right\}  \leq  \exp\left(-\frac{v^2}{L_1^2} \cdot 4L_1^2\ln\frac{2T\N(\F, 1/T)^2}{\delta}\right)  \\ \leq& \exp\left( -\ln\frac{2T\N(\F, 1/T)^2}{\delta}\right)
        \le \frac{\delta}{2T\N(\F, 1/T)^2}.
    \end{align*}
    By a union bound over all such function pairs $(f,f')$, this implies that with probaiblity at least $1-\delta/(2T)$, for any $f, f' \in \C(\F, 1/T)$, 
    \begin{align*}
        T_{m-1}\E_{s\sim \D_h(\pi^{m-1}), a\sim\pi_h^{m-1}(s)} \left[  (f(s, a)-f'(s, a))^2\right] \leq 2 \Vert f-f' \Vert_{\Z_h^m}^2.
    \end{align*}
    Now we condition on the event above in the following part of the proof. 
    
    To simplify the notation, we denote 
    \begin{align*}
        \Vert f-f'\Vert_{\pi_h^{m-1}}^2 = \E_{s\sim \D_h(\pi^{m-1}), a\sim\pi_h^{m-1}(s)} \left[  (f(s, a)-f'(s, a))^2\right], \forall f, f' \in \F.
    \end{align*}
    For any pair of functions $f, f' \in \F$, there exists $\hat{f}, \hat{f}' \in \C(\F, 1/T)$, s.t.  $\Vert f-\hat{f} \Vert_\infty \leq 1/T$ and  $\Vert f'-\hat{f}' \Vert_\infty \leq 1/T$. 
    When $\Vert f-f'\Vert_{\pi_h^{m-1}}^2 \leq 64/T_{m-1}$, we can directly obtain that 
    \begin{align*}
        T_{m-1}\Vert f-f'\Vert_{\pi_h^{m-1}}^2 \leq 4 \Vert f-f' \Vert_{\Z_h^m}^2 + 64.
    \end{align*}
    So we only consider the case when $\Vert f-f'\Vert_{\pi_h^{m-1}}^2 \geq 64/T_{m-1}$. Then, we have 
    \begin{align*}
        \Vert f-f' \Vert_{\Z_h^m} \geq&  \Vert \hat{f} - \hat{f}' \Vert_{\Z_h^m} - \Vert f - \hat{f} \Vert_{\Z_h^m} - \Vert f' - \hat{f}' \Vert_{\Z_h^m}  \\ \geq& \sqrt{T_{m-1}/2}\Vert \hat{f}-\hat{f}'\Vert_{\pi_h^{m-1}} - 2/\sqrt{T} \\ \geq&  \sqrt{T_{m-1}/2}\left(\Vert f-f'\Vert_{\pi_h^{m-1}} - \Vert f-\hat{f} \Vert_{\pi_h^{m-1}} - \Vert f'-\hat{f}' \Vert_{\pi_h^{m-1}} \right) - 2/\sqrt{T}  \\ \geq& 
         \sqrt{T_{m-1}/2}\left(\Vert f-f'\Vert_{\pi_h^{m-1}} - 2/T \right) - 2/\sqrt{T}
        \\ \geq& 
      \sqrt{T_{m-1}/2}\Vert f-f'\Vert_{\pi_h^{m-1}} - 4/\sqrt{T} \geq 0.
    \end{align*}
    Therefore, 
    \begin{align*}
        \Vert f-f' \Vert_{\Z_h^m}^2 \geq& (T_{m-1}/4)\cdot\Vert f-f'\Vert_{\pi_h^{m-1}}^2 - 16/T \geq  (T_{m-1}/4)\cdot\Vert f-f'\Vert_{\pi_h^{m-1}}^2 - 16, 
    \end{align*}
    which means 
    \begin{align*}
         T_{m-1}\E_{s\sim \D_h(\pi^{m-1}), a\sim\pi_h^{m-1}(s)} [(f(s,a) - f'(s,a))^2 ]\leq 4\sum_{k=\tau_{m-1}}^{\tau_m-1} (f(s_h^k,a_h^k)-f'(s_h^k,a_h^k))^2 + 64.
    \end{align*}
    Finally, we complete the proof by directly applying a union bound over all $(h,m) \in [H]\times([M]\backslash[M_0])$.
    
\end{proof}

Now we are ready to prove the main theorem. 

\begin{proof}[Proof of \Cref{thm_main}]
    We condition on the event defined in \Cref{lem_confidence_region}, \Cref{lem_optimistic_Q}, \Cref{lem_regret_decomp} and \Cref{lem_bound_of_square_error}. Also, we condition on the event in \Cref{prop_stable_sampling}
    after applying a union bound over all $(h,m) \in [H]\times([M]\backslash[M_0])$. With probability at least $1-\delta$, all the above events hold.
    
    By \Cref{lem_regret_decomp}, we have 
    \begin{align*}
        \Reg(K) \leq \tau_{M_0+1}\cdot H + 2\sum_{m = M_0+1}^M \sum_{k=\tau_m}^{\tau_{m+1}-1} \sum_{h=1}^H b_h^m(s_h^k, a_h^k) + 8H\sqrt{T\ln(16/\delta)}.
    \end{align*}
    For any $(h,m) \in [H]\times([M]\backslash[M_0])$, we define 
    \begin{align*}
        \overline{\F}_h^m = \left\{ f \in \F \ | \   \Vert f-f_h^m \Vert_{\Z^m}^{2} \leq 12\beta(\F, \delta) + 12 \right\}, 
    \end{align*}
    where 
    \begin{align*}
    \Z^m = \left\{ (s_h^{k}, a_h^{k}) \right\}_{(h,k)\in[H]\times[\tau_m-1]}
\end{align*}
as defined in \Cref{alg_main}. 
Let 
\begin{align*}
    \High_{\overline{\F}_h^m}(s,a) = \max_{f\in\overline{\F}_h^m} f(s,a), \quad \Low_{\overline{\F}_h^m}(s,a) = \min_{f\in\overline{\F}_h^m} f(s,a).
\end{align*}

    By \Cref{prop_stable_sampling}, $b_h^m(\cdot, \cdot) \leq w(\overline{\F}_h^m, \cdot,\cdot)$. Then, for any  episode $k \in [\tau_m, \tau_{m+1}-1]$, 
    \begin{align*}
        & \left(b_h^m(s_h^k, a_h^k)\right)^2 \leq \left(w(\overline{\F}_h^m, s_h^k, a_h^k)\right)^2 \leq \left(\High_{\overline{\F}_h^m}(s_h^k,a_h^k) -  \Low_{\overline{\F}_h^m} ( s_h^k, a_h^k)\right)^2 \\ \leq& \left(\High_{\overline{\F}_h^m}(s_h^k,a_h^k) - f_h^m(s_h^k,a_h^k) + f_h^m(s_h^k,a_h^k) -  \Low_{\overline{\F}_h^m} ( s_h^k, a_h^k)\right)^2 \\ 
        \leq& 2\left(\High_{\overline{\F}_h^m}(s_h^k,a_h^k) - f_h^m(s_h^k,a_h^k) \right)^2
        + 2\left( f_h^m(s_h^k,a_h^k) -  \Low_{\overline{\F}_h^m} ( s_h^k, a_h^k)\right)^2 
        \\ \leq& 4 \sup_{f \in \overline{\F}_h^m } \left( f(s_h^k,a_h^k) - f_h^m(s_h^k,a_h^k)\right)^2 \\ \leq& 4 L_1 \sup_{f \in \overline{\F}_h^m } \E_{s \sim \D_h(\pi^{m-1}) } \E_{a \sim \pi_h^{m-1}(s)} \left[  (f(s, a) - f_h^m(s, a))^2\right] \\ 
        \overset{\text{\Cref{lem_bound_of_square_error}}}{\leq}& \frac{4 L_1}{T_{m-1}} \cdot \sup_{f \in \overline{\F}_h^m } \left( 4\sum_{{k'}=\tau_{m-1}}^{\tau_m-1} (f(s_h^{k'},a_h^{k'})-f_h^m(s_h^{k'},a_h^{k'}))^2 + 64 \right)
        \\ 
        \leq& \frac{4 L_1}{T_{m-1}} \cdot \sup_{f \in \overline{\F}_h^m } \left( 4\Vert f-f_h^m \Vert_{\Z^m}^{2} + 64 \right) \\ 
        \leq& \frac{4 L_1}{T_{m-1}} \cdot  \left( 4 \times (12\beta(\F, \delta) + 12) + 64 \right) 
        \\  =& \frac{64 L_1}{T_{m-1}} \cdot  \left( 3\beta(\F, \delta) + 7 \right).
    \end{align*}
    Therefore, 
    \begin{align*}
        &\left(\sum_{m = M_0+1}^M \sum_{k=\tau_m}^{\tau_{m+1}-1} \sum_{h=1}^H b_h^m(s_h^k, a_h^k)\right)^2 \\ 
        \leq& \left(\sum_{m = M_0+1}^M \sum_{k=\tau_m}^{\tau_{m+1}-1} \sum_{h=1}^H \left(b_h^m(s_h^k, a_h^k)\right)^2\right)\cdot T
        \\
        \leq& 64TL_1 \sum_{m = M_0+1}^M \sum_{k=\tau_m}^{\tau_{m+1}-1} \sum_{h=1}^H \frac{ 3\beta(\F, \delta) + 7}{T_{m-1}} \\ \leq& 128TL_1 HM(3\beta(\F, \delta) + 7),
    \end{align*}
    which implies
    \begin{align*}
        &2\sum_{m = M_0+1}^M \sum_{k=\tau_m}^{\tau_{m+1}-1} \sum_{h=1}^H b_h^m(s_h^k, a_h^k) \leq 32\sqrt{L_1THM(3\beta(\F, \delta) + 7)}.
    \end{align*}
    Then, we can obtain that 
    \begin{align*}
        &\Reg(K) \\  \leq& 2^{M_0}\cdot H + 32\sqrt{L_1THM(3\beta(\F, \delta) + 7)} + 8H\sqrt{T\ln(16/\delta)} \\
        \leq& 64L_1^2H\ln\frac{128T\N(\F, \delta/(9216T^2))^2}{\delta}  + 32\sqrt{L_1THM(3\beta(\F, \delta) + 7)} + 8H\sqrt{T\ln(16/\delta)}\\ \leq& O(L_1^2H (\ln(T/\delta) + \ln(\N(\F,\delta/T^2) ))) \\ &+   O(L_1^2H^{3/2}\ln^2(T/\delta)\cdot\max(\ln(\N(\F,\delta/T^3)),\ln(\N(\S\times\A,\delta/T^2)))\cdot\sqrt{T})
        \\ \leq& O(L_1H^{3/2}\ln^2(T/\delta)\cdot\max(\ln(\N(\F,\delta/T^3)),\ln(\N(\S\times\A,\delta/T^2)))\cdot\sqrt{T}) 
    \end{align*}
\end{proof}

\section{Proof of \Cref{prop_surprise_bound_special_case}}
\label{sec_proof_of_surprise_bound_prop}

In this section, we provide the proof of \Cref{prop_surprise_bound_special_case}.

\begin{proof}[Proof of \Cref{prop_surprise_bound_special_case}]
For linear settings, let $\W^* = \{w-w' | w,w' \in\W \}$, then by \Cref{defn_vsb},
    \begin{align*}
        L_1 \leq& \sup_\pi \max_{h\in[H]} \sup_{w,w' \in\W} \frac{\sup_{(s,a)\in\S\times\A} (w^\T\phi(s,a) - {w'}^{\T}\phi(s,a))^2}{\E_{s' \sim \D_h(\pi) } \E_{a' \sim \pi_h(s')} \left[  (w^\T\phi(s', a') - {w'}^{\T}\phi(s', a'))^2\right]}
        \\
        \leq& \sup_\pi \max_{h\in[H]} \sup_{w \in\W^*} \frac{\sup_{(s,a)\in\S\times\A} (w^\T\phi(s,a))^2}{\E_{s' \sim \D_h(\pi) } \E_{a' \sim \pi_h(s')} \left[  (w^\T\phi(s', a'))^2\right]} \\  
        \leq& \sup_\pi \max_{h\in[H]} \sup_{w \in\W^*} \frac{\Vert w \Vert_2^2}{\E_{s' \sim \D_h(\pi) } \E_{a' \sim \pi_h(s')} \left[  (w^\T\phi(s', a'))^2\right]}
        \\ \leq& \sup_\pi \max_{h\in[H]} \sup_{w \in\W^*} \frac{\Vert w \Vert_2^2}{w^\T\E_{s' \sim \D_h(\pi) } \E_{a' \sim \pi_h(s')} \left[  \phi(s', a')\phi(s', a')^\T\right]w}
        \\ \leq& \sup_\pi \max_{h\in[H]} \sup_{w \in\W^*} \frac{\Vert w \Vert_2^2}{\Vert w \Vert_2^2 \lambda_{\min}\left(\E_{s' \sim \D_h(\pi) } \E_{a' \sim \pi_h(s')} \left[  \phi(s', a')\phi(s', a')^\T\right]\right)} \\ \leq& \sup_\pi \max_{h\in[H]}  \frac{1}{ \lambda_{\min}\left(\E_{s' \sim \D_h(\pi) } \E_{a' \sim \pi_h(s')} \left[  \phi(s', a')\phi(s', a')^\T\right]\right)}. 
    \end{align*}
    
    For sparse high-dimensional linear settings, let $\W^* = \{w-w' | w,w' \in\W \}$, then by \Cref{defn_vsb},
    \begin{align*}
        L_1 \leq& \sup_\pi \max_{h\in[H]} \sup_{w,w' \in\W} \frac{\sup_{(s,a)\in\S\times\A} (w^\T\phi(s,a) - {w'}^{\T}\phi(s,a))^2}{\E_{s' \sim \D_h(\pi) } \E_{a' \sim \pi_h(s')} \left[  (w^\T\phi(s', a') - {w'}^{\T}\phi(s', a'))^2\right]} \\
        \leq& \sup_\pi \max_{h\in[H]} \sup_{w \in\W^*} \frac{\sup_{(s,a)\in\S\times\A} (w^\T\phi(s,a))^2}{\E_{s' \sim \D_h(\pi) } \E_{a' \sim \pi_h(s')} \left[  (w^\T\phi(s', a'))^2\right]} \\  
        \leq& \sup_\pi \max_{h\in[H]} \sup_{w \in\W^*} \frac{4s\Vert w \Vert_2^2}{\E_{s' \sim \D_h(\pi) } \E_{a' \sim \pi_h(s')} \left[  (w^\T\phi(s', a'))^2\right]}
        \\ \leq& \sup_\pi \max_{h\in[H]} \sup_{w \in\W^*} \frac{4s\Vert w \Vert_2^2}{w^\T\E_{s' \sim \D_h(\pi) } \E_{a' \sim \pi_h(s')} \left[  \phi(s', a')\phi(s', a')^\T\right]w}
        \\ \leq& \sup_\pi \max_{h\in[H]} \sup_{w \in\W^*} \frac{4s\Vert w \Vert_2^2}{\Vert w \Vert_2^2 \psi_{\min}\left(\E_{s' \sim \D_h(\pi) } \E_{a' \sim \pi_h(s')} \left[  \phi(s', a')\phi(s', a')^\T\right]\right)} \\ \leq& \sup_\pi \max_{h\in[H]}  \frac{4s}{ \psi_{\min}\left(\E_{s' \sim \D_h(\pi) } \E_{a' \sim \pi_h(s')} \left[  \phi(s', a')\phi(s', a')^\T\right]\right)}.
    \end{align*}
\end{proof}

\section{Proof of \Cref{thm_misspecification}} 
\label{appendix_proof_of_misspecification}

In this section, we provide the proof of \Cref{thm_misspecification} for model misspecification. First, we slightly modify \Cref{lem_single_step_error} and reprove it in model misspecification case.

\begin{lemma}[Single step optimization error for misspecification] 
\label{lem_single_step_error_misspecification}

Assume that our function class $\F$ satisfies \Cref{assump_misspecification}. Consider a fixed epoch $m \in [M]\backslash[M_0]$. We define 
\begin{align*}
    \Z^m = \left\{ (s_h^{k}, a_h^{k}) \right\}_{(h,k)\in[H]\times[\tau_m-1]}
\end{align*}
as in \Cref{alg_main}. Also, for any function $V: \S \to [0,H]$, we define
\begin{align*}
    \D_V^m = \left\{ \left(s_{h}^{k}, a_{h}^{k}, r_{h}^{k} + V(s_{h+1}^{k})\right) \right\}_{(h,k)\in[H]\times[\tau_m-1]}
\end{align*}
and
\begin{align*}
    \hat{f}_V = \arg\min_{f \in \F} \Vert f \Vert_{\D_V^m}^2.
\end{align*}
Then, for  any function $V: \S \to [0,H]$ and $\delta \in (0,1)$, there exists an event $\mathcal{E}_{V,\delta}$ where $\Pr\{\mathcal{E}_{V,\delta}\} \geq 1-\delta$, s.t. conditioned on $\mathcal{E}_{V,\delta}$, for any $V': \S \to [0,H]$ with $\Vert V-V' \Vert_\infty \leq 1/T$, we have 
\begin{align*}
    \left\Vert \hat{f}_{V'}(\cdot, \cdot) - r(\cdot, \cdot) - \sum_{s'\in\S} P(s'|\cdot,\cdot)V'(s') \right\Vert_{\Z^m}  \leq c'\sqrt{H^2(\ln(T/\delta) +\ln \N(\F, 1/T)) + HT\zeta}.  
\end{align*}
for some constant $c' > 0$.
\end{lemma}

\begin{proof}
    For any $V: \S \to [0,H]$, we define 
    \begin{align*}
        f_V(\cdot, \cdot) = r(\cdot, \cdot) + \sum_{s'\in \S} P(s' | \cdot, \cdot)V(s'),
    \end{align*}
    and now we consider a fixed $V$. Note that under \Cref{assump_misspecification}, it does not necessary hold that $f_V \in \F$, but it can be ensured that 
    \begin{align*}
        \min_{f\in\F} \Vert f - f_V \Vert_{\Z^m}^2 \leq |\Z^m| \zeta^2 \leq T\zeta^2.  
    \end{align*}
    
    For any $f \in \F$, define 
    \begin{align*}
        \xi_h^{k}(f) = 2(f(s_h^k, a_h^k) - f_V(s_h^k, a_h^k) )\cdot(f_V(s_h^k, a_h^k) - r_h^k - V(s_{h+1}^k)), \,\, \forall (h, k) \in [H]\times[\tau_m-1].
    \end{align*}
    By the same method as in \Cref{lem_single_step_error}, we can prove that with probability at least $1-\delta$, 
\begin{align*}
    \left| \sum_{(h, k) \in [H]\times[\tau_m-1]} \xi_h^k(f) \right| \lesssim 8(H+1)^2\log\frac{2T+2}{\delta} +  4(H+1)\|f-f_V\|_{\mathcal{Z}^m} \sqrt{\log\frac{2T+2}{\delta}}.
\end{align*}
Let $\mathcal{E}_{V,\delta}$ denote the above event, and for the rest of the proof, we condition on $\mathcal{E}_{V,\delta}$.

Similarly, by the same method as in \Cref{lem_single_step_error}, for any $f \in \F$, we have
\begin{align*}
        &\left| \sum_{(h, k) \in [H]\times[\tau_m-1]} \xi_h^k(f) \right| \\ \lesssim& H^2(\ln (T/\delta)  + \ln \mathcal{N}(\mathcal{F}, 1/T)) +  H\|f-f_V\|_{\mathcal{Z}^m} \sqrt{\ln (T/\delta) + \ln \mathcal{N}(\mathcal{F}, 1/T)}.
\end{align*}

For any $V': \S \to [0,H]$ with $\Vert V'-V \Vert_\infty \leq 1/T$, we can obtain that 
\begin{align*}
    \Vert f_{V'}-f_V \Vert_\infty  = \left\Vert\sum_{s' \in \S} P(s'|\cdot, \cdot)(V'(s')-V(s')) \right\Vert_\infty\leq \Vert V'-V \Vert_\infty \leq 1/T.
\end{align*}
Furthermore, again by the same method as in \Cref{lem_single_step_error}, we can obtain that for any $f \in \F$, 
\begin{align*}
    &\Vert f\Vert_{\D_{V'}^m}^2 - \Vert f_{V'}\Vert_{\D_{V'}^m}^2 \\
    \gtrsim& \Vert f - f_{V'} \Vert_{\Z^m}^2   - H^2(\ln (T/\delta)  + \ln \mathcal{N}(\mathcal{F}, 1/T)) - H\|f-f_{V'}\|_{\mathcal{Z}^m} \sqrt{\ln (T/\delta) + \ln \mathcal{N}(\mathcal{F}, 1/T)}.
\end{align*}
If we let $f = \hat{f}_{V'} = \arg\min_{f\in\F} \Vert f\Vert_{\D_{V'}^m}$, we have
\begin{align*}
 &\Vert \hat{f}_{V'}\Vert_{\D_{V'}^m}^2 - \Vert f_{V'}\Vert_{\D_{V'}^m}^2  \\ \gtrsim& \Vert \hat{f}_{V'} - f_{V'} \Vert_{\Z^m}^2  - H^2(\ln (T/\delta)  + \ln \mathcal{N}(\mathcal{F}, 1/T)) - H\|\hat{f}_{V'}-f_{V'}\|_{\mathcal{Z}^m} \sqrt{\ln (T/\delta) + \ln \mathcal{N}(\mathcal{F}, 1/T)}.
\end{align*}
Now let $\Tilde{f}_{V'} = {\arg\min}_{f\in\F} \Vert f-f_{V'}\Vert_{\Z^m}^2$, then 
\begin{align*}
    & \Vert \hat{f}_{V'}\Vert_{\D_{V'}^m} \leq \Vert \Tilde{f}_{V'}\Vert_{\D_{V'}^m} \leq \Vert f_{V'}\Vert_{\D_{V'}^m} + \Vert f_{V'} - \Tilde{f}_{V'}\Vert_{\Z^m} \leq \Vert f_{V'}\Vert_{\D_{V'}^m} + \sqrt{T}\zeta \\
    \Longrightarrow & \Vert \hat{f}_{V'}\Vert_{\D_{V'}^m} - \Vert f_{V'}\Vert_{\D_{V'}^m} \leq \sqrt{T}\zeta 
    \\ \Longrightarrow & \Vert \hat{f}_{V'}\Vert_{\D_{V'}^m}^2 - \Vert f_{V'}\Vert_{\D_{V'}^m}^2 \leq \sqrt{T}\zeta(\Vert \hat{f}_{V'}\Vert_{\D_{V'}^m} + \Vert f_{V'}\Vert_{\D_{V'}^m}) \leq \sqrt{T}\zeta \cdot 4\sqrt{T}H = 4HT\zeta.
\end{align*}

Therefore, 
\begin{align*}
    &\Vert \hat{f}_{V'} - f_{V'} \Vert_{\Z^m}^2 \\ \lesssim& H^2(\ln (T/\delta)  + \ln \mathcal{N}(\mathcal{F}, 1/T)) + H\|\hat{f}_{V'}-f_{V'}\|_{\mathcal{Z}^m} \sqrt{\ln (T/\delta) + \ln \mathcal{N}(\mathcal{F}, 1/T)} + 4HT\zeta.
\end{align*}

which implies
\begin{align*}
    \Vert \hat{f}_{V'} - f_{V'} \Vert_{\Z^m} \leq c'\sqrt{H^2(\ln(T/\delta)+\ln\N(\F,1/T))+ HT\zeta}.
\end{align*}
for some constant $c' > 0$.
\end{proof}

Using the above lemma, we can obtain the following lemma similar to \Cref{lem_confidence_region}.

\begin{lemma}[Confidence region for misspecification]
\label{lem_confidence_region_misspecification}
Assume that our function class $\F$ satisfies \Cref{assump_misspecification}. 
In \Cref{alg_main}, for $m > M_0$, define confidence region
\begin{align*}
    \F_h^m = \left\{ f \in \F \left| \Vert f - f_h^m \Vert_{\Z^m}^2 \leq \beta(\F, \delta) \right. \right\}.
\end{align*}
Then with probability at least $1-\delta/16$, for all $(h, m) \in [H]\times([M]\backslash[M_0])$,
\begin{align*}
   r(\cdot, \cdot) + \sum_{s'\in \S} P(s' | \cdot, \cdot)V_{h+1}^m(s') \in \F_h^m,
\end{align*}
given
\begin{align*}
    \beta(\F, \delta) \geq c'(H^2(\ln(T/\delta) +\ln \N(\F, 1/T) + \ln |\W|) + HT\zeta).
\end{align*}
for some constant $c' > 0$. Here, $\W$ is given in \Cref{prop_stable_sampling}.
\end{lemma}

\begin{proof}
    The proof is almost identical to that of \Cref{lem_confidence_region}.
\end{proof}

\begin{proof}[Proof of \Cref{thm_misspecification}]
    By \Cref{lem_confidence_region_misspecification}, \Cref{lem_optimistic_Q}, \Cref{lem_regret_decomp}, \Cref{lem_bound_of_square_error}, the proof is almost the same as the proof of \Cref{thm_main}.
\end{proof}

\end{document}